\newif\ifarXiv
\arXivtrue

\newif\ifanonymous
\anonymoustrue

\ifarXiv
    \anonymousfalse
\fi

\RequirePackage{fix-cm}
\documentclass{article}

\ifarXiv
    \usepackage{natbib}
    \usepackage[margin=2cm]{geometry}
\else
    \usepackage{iclr2027_conference,times}
\fi

\usepackage[T1]{fontenc}
\usepackage[utf8]{inputenc}
\usepackage{microtype}
\usepackage{amsmath,amssymb,amsthm,mathtools}
\usepackage{booktabs,tabularx,array,multirow}
\usepackage{graphicx}
\usepackage{subcaption}
\usepackage{xcolor}
\usepackage{enumitem}
\usepackage{verbatim}
\usepackage[hidelinks]{hyperref}
\usepackage[nameinlink,noabbrev,capitalize]{cleveref}
\usepackage{float}
\usepackage{placeins}

\usepackage{caption}
\usepackage{siunitx}
\usepackage{longtable}
\usepackage{pdflscape}
\usepackage{tikz}
\usetikzlibrary{arrows.meta,calc,positioning}
\usepackage{todonotes}
\setuptodonotes{inline, color=yellow!80!red}
\usepackage[most]{tcolorbox}
\usepackage{silence}
\definecolor{takeawayline}{HTML}{009E73}
\colorlet{takeawaybg}{takeawayline!10!white}

\newtcolorbox{takeaway}{
  colback=takeawaybg,
  colframe=takeawayline,
  boxrule=0.5mm,
  arc=1mm,
  outer arc=1.5mm,
  left=5pt,
  right=5pt,
  top=5pt,
  bottom=5pt,
  boxsep=0pt,
  before skip=6pt,
  after skip=8pt,
  fontupper=\normalsize
}

\newcommand{\softmax}{\operatorname{softmax}}

\newtheorem{proposition}{Proposition}
\newtheorem{lemma}{Lemma}

\title{Propagate, Then Sharpen: Post-Hoc Refinement of Frozen Node Classifiers}

\ifanonymous
    \author{Anonymous Authors}
\else
    \author{
        Preben Johnsen Bentdal \\
        Department of Informatics, University of Bergen \\
        Norwegian Research Centre (NORCE) \\
        \texttt{pbe027@uib.no}
        \and
        Nello Blaser \\
        Department of Informatics, University of Bergen \\
        \texttt{nello.blaser@uib.no}
        \and
        Xue-Cheng Tai \\
        Norwegian Research Centre (NORCE) \\
        \texttt{xtai@norceresearch.no}
    }
    \date{}
\fi

\begin{document}
\maketitle
\raggedbottom

\begin{abstract}
We study post-hoc refinement of frozen node classifiers: given only the graph $G$ and class distributions $Q$ predicted by a frozen model, can we improve accuracy without access to node features, model parameters, or gradients? 
APPNP answers this by propagating logits with a restart towards the initial predictions, minimizing the anchored Dirichlet energy.
Instead, we consider the Potts energy, and decompose it into a Dirichlet term, which penalizes disagreement between neighbouring nodes, and a Gini term, which penalizes indecision within each node.
This decomposition motivates \emph{Propagate, Then Sharpen} (PtS), which alternates between propagation of class probabilities and node-wise, mass-preserving sharpening, with only one additional hyperparameter selected using labelled validation nodes.
Across nine homophilic graphs, with a frozen MLP backbone, PtS improves mean test accuracy over independently tuned APPNP by $1.71$ percentage points on clean inputs and $3.90$ under severe Gaussian feature corruption.
Gains over APPNP become smaller, but remain positive with frozen GCN and GraphSAGE backbones.
Sharpening also removes most of the accuracy loss of deep propagation: on clean inputs without restart, accuracy falls by $2.2$ points between $2$ and $100$ propagation steps under PtS, compared with $33.8$ for APPNP.
\end{abstract}

\section{Introduction}

Graph-based node classification uses node attributes (features) and relationships between nodes (structure).
These two inputs are not equally reliable. Features can be fragile: sensors drift or fail, attributes go missing or stale for part of the graph, or features are perturbed to protect privacy, while the graph structure usually stays intact.
When this happens, correcting the model is often not an option, because retraining may be costly or access to parameters or training data might be restricted. These constraints motivate post-hoc methods that reuse the deployed model as a frozen predictor~\citep{boudiaf2022lame}.
We ask how much prediction accuracy can be recovered using only the frozen class distributions $Q$ and the graph $G$, without access to node features, model parameters, gradients, or training data.
The refinement should apply to any backbone, so that another predictor can easily be swapped in. Validation labels should only be used for hyperparameter selection, and refinement should tolerate a mismatch between validation and test conditions.
We therefore evaluate refinement both on clean inputs, and under Gaussian feature corruption on datasets with continuous-valued node features for which additive Gaussian noise provides a meaningful controlled perturbation.

A natural way to recover accuracy on $G$ is to propagate predictions along its edges, and APPNP is the standard way to do so: personalized PageRank propagation of the logits, with a restart towards the original prediction at every step \citep{gasteiger2019appnp}.
Propagation exploits informative neighbours, but repeated propagation also removes distinctions between nodes, which can result in oversmoothing \citep{oono2020oversmoothing}, i.e. predictions in a connected component converge towards a common value as the number of propagation steps grows.
Restart is what prevents this collapse in APPNP, but it only helps if the reference prediction is reliable. If features are missing or corrupted at test time, the restart repeatedly re-injects corrupted predictions, and propagation spreads these errors through the graph.
APPNP is therefore sensitive to its depth and restart hyperparameters, and without restart it collapses at large depth (\cref{fig:depth}).

The space in which predictions are propagated matters. APPNP propagates logits, so large class margins exert great influence, which helps when margins are reliable but amplifies confident errors.
Propagating probabilities instead bounds the initial scores in $[0,1]$, limiting the influence of extreme logits. The trade-off is that averaging neighbouring distributions can weaken class preferences when neighbours disagree.
This motivates a refinement method that acts in two directions: smoothing predictions across neighbouring nodes, and sharpening the class distribution within each node. 
We start from the relaxed Potts interaction used in variational image segmentation \citep{potts1952generalized,liu2022std,tai2024pottsmgnet,liu2024doublewellnet}. 
On the probability simplex, the interaction splits exactly into a Dirichlet term, which penalizes disagreement between neighbouring nodes, and a Gini term, which penalizes indecision within each node. 
This decomposition pairs graph propagation with a local sharpening step that reduces indecision within each node.

We call the resulting method Propagate, Then Sharpen (PtS). It alternates between a propagation and a sharpening step.
Sharpening is controlled by a single parameter $\eta$: at $\eta=0$ we recover propagation without sharpening exactly, so validation can switch sharpening off, as $\eta \to \infty$ sharpening approaches a one-hot assignment when the most probable class is unique, and in between it acts as a soft thresholding.
Sharpening never changes a node's own predicted class, it changes the distribution used by the next propagation step.
PtS uses only $Q$ and $G$, labels enter only through hyperparameter selection, any backbone that outputs class probabilities can be used, and its gain over APPNP persists at large propagation depth (\cref{fig:depth}), as well as when hyperparameters are selected at a different corruption severity than the one evaluated at (\cref{fig:transfer}).

\paragraph{Contributions}
\begin{itemize}[leftmargin=*,nosep]
    \item  We introduce Propagate, Then Sharpen (PtS), which alternates probability-space graph propagation with a mass-preserving categorical sharpening step derived from the Dirichlet--Gini decomposition of the Potts interaction. Sharpening adds one hyperparameter \(\eta\), is off at $\eta=0$, provably preserves every node's predicted class, and decreases a local KL--Gini energy at each step (\cref{prop:descent}).

    \item  We show that sharpening reduces accuracy degradation with propagation depth. On a constructed two-community graph, restart alone fails while PtS keeps every node correct at every depth (\cref{prop:two-community}).  Across nine homophilic graphs on clean inputs without restart, increasing depth from $K=2$ to $K=100$ reduces mean accuracy by $5.2$ percentage points for PtS and $22.0$ for Logit-Sharp (APPNP with sharpening in  logit-space) at the same sharpening parameter $\eta=16$. Increasing PtS's sharpening parameter to $\eta=200$ reduces its drop to $2.2$ points (\cref{fig:depth}).

    \item On nine homophilic graphs, PtS gains $3.90$ points over independently tuned APPNP under severe Gaussian corruption and $1.71$ on clean inputs (\cref{tab:repair}), and improves on average over a graph-adapted LAME and a post-hoc Graph-TV baseline (\cref{tab:external_baselines,tab:baselines-clean}).
\end{itemize}

\paragraph{Related Work}
The sharpening step of PtS comes from variational image segmentation, where the Potts model penalizes the boundaries between regions of a soft assignment field. PottsMGNet and Double-well Net connect this formulation with trainable neural architectures by deriving their updates from Potts energies and operator-splitting schemes \citep{tai2024pottsmgnet,liu2024doublewellnet}, and soft threshold dynamics adds spatial regularization and shape priors to the output of a segmentation network \citep{liu2022std}. We transfer the same splitting idea from the image domain to a graph, where the assignment field becomes the class distributions of the nodes and the split separates smoothing from local sharpening. The combination already appears on graphs in diffuse-interface methods, which pair smoothing with phase separation, and in Merriman--Bence--Osher (MBO) schemes, which alternate diffusion and thresholding \citep{bertozzi2012diffuse,merkurjev2013mbo,garcia2014diffusegraph}. PtS is a soft, mass-preserving counterpart applied to frozen predictions rather than labels, approaching hard thresholding as $\eta\to\infty$ when the largest class probability is unique. Allen--Cahn Message Passing (ACMP) and GREAD apply reaction--diffusion dynamics to learned node representations inside trained graph neural networks \citep{wang2025acmpallencahnmessagepassing,choi2023gread}, PtS acts on the outputs of an already trained model and has no learnable parameters. \citet{yang2025regularizing} incorporate a graph-based non-local TV-regularized softmax into GNN training. We evaluate its objective as a post-hoc baseline on frozen predictions,
denoted Graph-TV (\cref{app:external_baselines}).

Post-hoc propagation methods refine predictions by spreading information through the graph structure. PPNP and APPNP combine neural prediction with personalized PageRank propagation in logit space, with APPNP using an iterative approximation with restart \citep{gasteiger2019appnp}. Correct \& Smooth applies label propagation after training, using known training labels to correct prediction errors and smooth the corrected predictions \citep{huang2021cs}. PtS uses no labels at inference, but its sharpening step can be inserted into the smoothing stage of Correct \& Smooth (\cref{app:correct_smooth}).
Laplacian Adjusted Maximum-Likelihood Estimation (LAME) refines a frozen classifier's output without updating its parameters. Its objective combines KL fidelity to the original class probabilities with a pairwise agreement term weighted by affinities computed from pretrained feature representations \citep{boudiaf2022lame}. Replacing that affinity with the graph operator turns its update into a fixed-point iteration of the full KL-plus-Potts energy (\cref{app:std}), which we evaluate as a further baseline (LAME-Graph). PtS alternates anchored propagation with local, mass-preserving sharpening, using the propagated distribution as the reference for each sharpening step.

\section{Method} 
\label{sec:setup}
Let $G=(\mathcal V,\mathcal E)$ be a graph with $N=|\mathcal{V}|$ nodes and let $C$ be the number of classes. 
We symmetrize directed source graphs, remove duplicate edges and existing self-loops.
Let $A\in\{0,1\}^{N\times N}$ denote the resulting binary adjacency without self-loops, $\widetilde A=A+I$ the adjacency with added self-loops, $\widetilde D$  the diagonal degree matrix $\widetilde D_{ii}=\sum_j \widetilde A_{ij}$, and $S$ the normalised adjacency 
\[
S=\widetilde D^{-1/2}\widetilde A\,\widetilde D^{-1/2}.
\]
The operator $S$ is symmetric, element-wise non-negative, and $\operatorname{spec}(S)\subseteq[-1,1]$ with $1$ an eigenvalue.

A frozen classifier $f_\theta$ produces logits $Z\in\mathbb{R}^{N\times C}$ and class distributions $Q=\softmax(Z)\in\mathbb{R}^{N\times C}$, with each row on the probability simplex $\Delta^{C-1}=\{p\in\mathbb{R}^{C}:p\ge0,\ \mathbf 1^{\top}p=1\}$. 

Refinement produces a sequence of states $U^{(0)}, \dots, U^{(K)} \in \mathbb{R}^{N\times C}$. We write $U^{(k)}_i\in\mathbb{R}^{C}$ for row $i$ of $U^{(k)}$, treated as a column vector. 

\begin{figure}[H]
    \centering
    \resizebox{0.95\linewidth}{!}{
\begingroup

\definecolor{ink}{HTML}{20242A}
\definecolor{muted}{HTML}{747B87}
\definecolor{hair}{HTML}{D9DDE3}
\definecolor{redc}{HTML}{C94A5A}
\definecolor{bluec}{HTML}{467AA5}
\definecolor{accent}{HTML}{BC7425}

\definecolor{methodAPPNP}{HTML}{2A78D6}
\definecolor{methodPPR}{HTML}{AF8724}
\definecolor{methodPtS}{HTML}{A33B57}

\tikzset{
  edge/.style={draw=hair,line width=.55pt},
  arrow/.style={
    -{Latex[length=2.2mm,width=1.35mm]},
    draw=ink,line width=.82pt
  },
  msg/.style={
    -{Latex[length=1.8mm,width=1.1mm]},
    draw=muted,line width=.68pt
  },
  stage/.style={
    font=\sffamily\bfseries\fontsize{8.5}{9.9}\selectfont,
    text=ink,align=center
  },
  tiny/.style={
    font=\sffamily\fontsize{5.9}{6.9}\selectfont,
    text=muted,align=center
  },
  state/.style={
    font=\sffamily\fontsize{7.2}{8.6}\selectfont,
    text=ink,align=center
  },
}

\newcommand{\graphbase}[3]{%
  \begin{scope}[shift={(#1,#2)}]
    \coordinate (n1) at (-1.12,.56);
    \coordinate (n2) at (-.55,.98);
    \coordinate (n3) at (-.73,-.04);
    \coordinate (n4) at (.00,.32);
    \coordinate (n5) at (.70,.06);
    \coordinate (n6) at (1.23,.64);
    \coordinate (n7) at (1.29,-.49);
    \coordinate (n8) at (.61,-.79);

    \foreach \a/\b in {
      1/2,1/3,2/3,2/4,3/4,4/5,5/6,5/7,6/7,6/8,7/8%
    }{
      \draw[edge] (n\a)--(n\b);
    }

    \foreach \i/\p in {#3}{
      \begin{scope}[shift={(n\i)}]
        \ifnum\i=4\relax
          \def\radius{.245}
        \else
          \def\radius{.225}
        \fi
        \pgfmathsetmacro{\endangle}{90+360*\p}
        \fill[bluec] (0,0) circle (\radius);
        \fill[redc] (0,0) -- (90:\radius)
          arc[start angle=90,end angle=\endangle,radius=\radius]
          -- cycle;
        \draw[white,line width=.6pt] (0,0) circle (\radius);
      \end{scope}
    }
  \end{scope}%
}

\newcommand{\pill}[4]{%
  \begin{scope}[shift={(#1,#2)}]
    \pgfmathsetmacro{\rw}{1.12*#3}
    \pgfmathsetmacro{\bw}{1.12-\rw}
    \fill[redc] (-.56,-.072) rectangle ++(\rw,.144);
    \fill[bluec] ({-.56+\rw},-.072) rectangle ++(\bw,.144);
    \draw[hair,line width=.4pt]
      (-.56,-.072) rectangle ++(1.12,.144);
    \node[tiny,anchor=north] at (0,-.11) {#4};
  \end{scope}%
}

\resizebox{\linewidth}{!}{%
\begin{tikzpicture}[x=1cm,y=1cm,line cap=round,line join=round]

\def\xA{-4.40}
\def\xB{0.00}
\def\xC{4.40}
\def\gy{-0.22}

\node[stage] at (\xA,1.65) {Frozen $Q=U^{(0)}$};
\node[stage,text=methodPPR] at (\xB,1.65) {Propagate on $G$};
\node[stage,text=methodPtS] at (\xC,1.65) {Sharpen each node};


\graphbase{\xA}{\gy}{
  1/.92,2/.86,3/.82,4/.45,5/.16,6/.08,7/.13,8/.10}
\coordinate (fa) at (\xA,{\gy+.32});
\draw[ink,line width=.8pt] (fa) circle (3.05mm);
\pill{\xA}{-1.48}{.45}{$45/55$}

\graphbase{\xB}{\gy}{
  1/0.874797,2/0.777220,3/0.773354,4/0.560250,
  5/0.200500,6/0.113287,7/0.118119,8/0.102825}
\coordinate (fb) at (\xB,{\gy+.32});
\draw[ink,line width=.8pt] (fb) circle (3.05mm);

\draw[msg,shorten <=2.5mm,shorten >=3.15mm]
  ({\xB-.55},{\gy+.98}) -- (fb);
\draw[msg,shorten <=2.5mm,shorten >=3.15mm]
  ({\xB-.73},{\gy-.04}) -- (fb);
\draw[msg,shorten <=2.5mm,shorten >=3.15mm]
  ({\xB+.70},{\gy+.06}) -- (fb);
\pill{\xB}{-1.48}{0.560250}{$56/44$}

\graphbase{\xC}{\gy}{
  1/0.989724,2/0.960459,3/0.958546,4/0.660142,
  5/0.029896,6/0.008454,7/0.009162,8/0.007058}
\coordinate (fc) at (\xC,{\gy+.32});
\draw[line width=.95pt] (fc) circle (3.25mm);
\pill{\xC}{-1.48}{0.660142}{$66/34$}

\draw[arrow]
  ({\xA+1.60},{\gy+.32}) -- ({\xB-1.60},{\gy+.32});
\draw[arrow]
  ({\xB+1.60},{\gy+.32}) -- ({\xC-1.60},{\gy+.32});

\node[state,anchor=south]
  at ({\xB-2.20},{\gy+.44}) {$U^{(k)}$};
\node[state,anchor=south]
  at ({(\xB+\xC)/2},{\gy+.44}) {$\widetilde U^{(k+1)}$};

\draw[arrow,rounded corners=2mm]
  ({\xC+1.60},{\gy+.32})
  -- ({\xC+1.85},{\gy+.32})
  -- ({\xC+1.85},-2.10)
  -- node[midway,below=3pt,state] {Repeat $K$ times}
     ({\xB-2.20},-2.10)
  -- ({\xB-2.20},{\gy+.32});

\node[state,fill=white,inner sep=2pt]
  at (\xC,-2.10) {$U^{(k+1)}$};
\fill[ink] ({\xB-2.20},{\gy+.32}) circle (.035);

\end{tikzpicture}%
}

\endgroup}
    \caption{\textbf{Propagate, Then Sharpen.} Node colours show two-class proportions in an illustrative first iteration ($\alpha =0.1$, $\eta =3.5$), initialized at $Q$. Propagation and sharpening update all nodes, and the sharpened state becomes the current state for the next iteration. Propagation includes self-loops and a restart towards the fixed $Q$. Pie charts show normalised class proportions.}
    \label{fig:overview}
\end{figure}
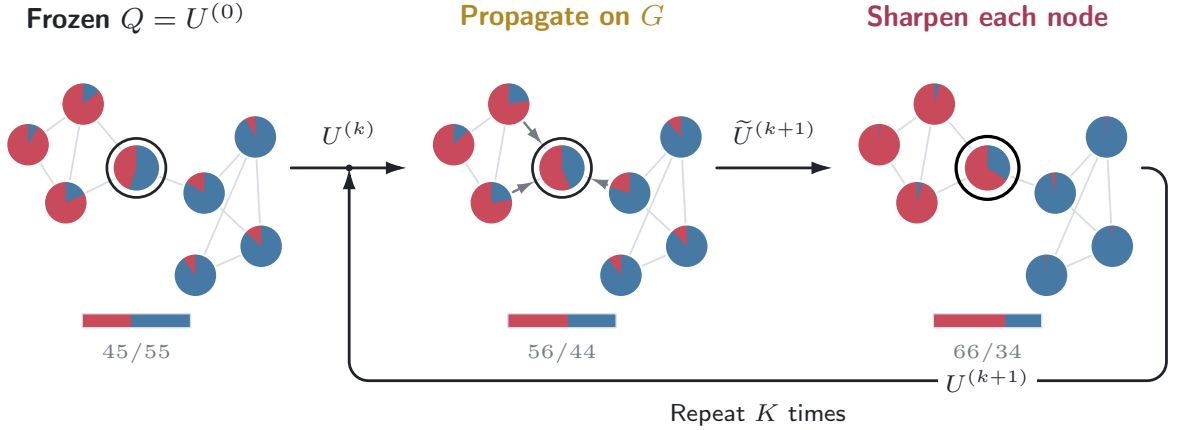

\subsection{Propagate} 
\label{TRANSPORT}

We start from the personalized PageRank recursion of APPNP \citep{gasteiger2019appnp}, but apply it to $Q$ instead of logits:
\begin{equation}
U^{(0)}=Q,\qquad
U^{(k+1)}=\mathcal T\bigl(U^{(k)}\bigr),
\qquad
\mathcal T(U)=\alpha Q+(1-\alpha)\,SU 
\label{eq:ppr-prob}
\end{equation}
with restart probability $\alpha\in[0,1]$ and $K$ propagation steps. The only difference from APPNP is the state being propagated: APPNP propagates the logits $Z$ and applies a softmax at the end, whereas this recursion propagates the class distributions $Q$. We call the variant \emph{PPR-Prob} and use it to isolate the effect of the propagation space. Since $S$ does not preserve unit row sums, we retain the unnormalized states during propagation and row-normalize $U^{(K)}$ to obtain the final class distributions. Normalising after every step instead is the PPR-rn variant compared in \cref{app:mass}.

\subsection{Propagate, Then Sharpen}
\label{sec:method}

PtS extends this propagation with a node-wise categorical sharpening step. In a reaction--diffusion interpretation, this local sharpening plays the role of the reaction step. As illustrated in \cref{fig:overview}, from $U^{(0)}=Q$, we alternate between
\begin{subequations}
\label{eq:PtS-update}
\begin{align}
\widetilde U^{(k+1)}
&=
\alpha Q+(1-\alpha)S U^{(k)}
&& \text{propagate}
\label{eq:PtS-transport}
\\[2pt]
U^{(k+1)}_i
&=
R_\eta\!\left(\widetilde U^{(k+1)}_i\right) \quad i = 1,\dots,N
&& \text{ sharpen}
\label{eq:PtS-sharpen}
\end{align}
\end{subequations}
where the reaction \(R_\eta\) is defined in \eqref{eq:reaction_with_m} below. 
Symmetric normalization does not preserve unit row sums. We therefore define the sharpening map on the simplex and then lift it to unnormalised rows. For $p \in \Delta^{C-1}$ with $p > 0$ and a reaction strength $\eta \ge 0$, let
\begin{equation}\label{eq:sharpen-simplex}
  r_\eta(p) \;=\; \softmax\!\left(\log p + \eta\, p\right),
\end{equation}
which maps $\Delta^{C-1}$ into itself. For a propagated row $h \in \mathbb{R}^{C}_{>0}$ with total mass $m(h) = \mathbf{1}^{\top} h$, the reaction is
\begin{equation}
\label{eq:reaction_with_m}
  R_\eta(h) \;=\; m(h)\;
    r_\eta\!\left(\frac{h}{m(h)}\right).
\end{equation}
At $\eta =0$ we recover $R_0(h) =h$, so PtS reduces exactly to PPR-Prob with the same $(\alpha,K)$.
Every row of $\widetilde U^{(k+1)}$ is strictly positive, so \eqref{eq:reaction_with_m} is well defined: $Q = \operatorname{softmax}(Z) > 0$ elementwise, and the added self-loops give $S$ a strictly positive diagonal, so positivity is preserved for every $\alpha \in [0,1]$.

Writing the reaction in the two stages of \eqref{eq:sharpen-simplex}--\eqref{eq:reaction_with_m} separates what it does. The map $r_\eta$ redistributes probability mass within a node, and the factor $m(h)$ restores the row total,
\begin{equation*}\label{eq:mass}
  \mathbf{1}^{\top} R_\eta(h) \;=\; m(h)\,\mathbf{1}^{\top} r_\eta\!\left(h/m(h)\right)
  \;=\; m(h) \;=\; \mathbf{1}^{\top} h .
\end{equation*}
The reaction therefore changes a node's class distribution while preserving the mass the row carries into the next propagation step, which is what determines the node's total contribution to its neighbours. We compare this choice with normalizing each row's mass after every iteration in \cref{tab:mass-preservation}. Since $\eta$ multiplies the normalised rows, the sharpening strength is invariant to the row total.


For every $\eta\geq0$, the reaction keeps the relative ranking of the classes within each row (\cref{app:sharpening}). It therefore never changes a node's own predicted class, and influences later predictions only through the state that enters the next propagation step. 

Without restart ($\alpha=0$) and without sharpening, propagation collapses: on a connected graph, every node converges to the same class distribution, and the predictions of APPNP converge to a single class shared by all nodes (\cref{app:collapse}). The experiments in \cref{fig:depth} show this collapse and how sharpening counteracts it.

\subsection{Energy interpretation and local descent}
\label{sec:theory}
Motivated by the entropy-regularized Potts formulations in
\citet{tai2024pottsmgnet}, consider normalised class distributions
$u_i\in\Delta^{C-1}$ and a symmetric affinity matrix $W$ with
$W_{ij}\geq0$ (for PtS, $W=S$). The relaxed Potts interaction is
\[
E(U)=\frac{1}{2}\sum_{i,j}W_{ij}
\left(1-u_i^\top u_j\right).
\]
For one-hot assignments, this is a weighted penalty for disagreement
between connected nodes. For soft assignments, it also penalizes
indecision within each node. Writing $d_i=\sum_j W_{ij}$ gives
\begin{equation*}
E(U)
=
\underbrace{
\frac{1}{4}\sum_{i,j}W_{ij}\|u_i-u_j\|^2
}_{E_{\mathrm D}(U)}
+
\underbrace{
\frac{1}{2}\sum_i d_i\left(1-\|u_i\|^2\right)
}_{E_{\mathrm G}(U)}.
\label{eq:potts-decomposition}
\end{equation*}
The Dirichlet term \(E_{\mathrm D}(U)\) penalizes differences between neighbouring distributions, while the Gini term \(E_{\mathrm G}(U)\) penalizes indecision within each node. This decomposition motivates the two operations of PtS. Propagation is a gradient step on an anchored, degree-normalised Dirichlet objective, while sharpening decreases a local KL--Gini energy. This does not establish descent of the joint Potts objective. \cref{fig:energy_evolve} illustrates how these components interact during propagation on WikiCS. 

Propagation descends an anchored, degree-normalised Dirichlet objective. Following the optimization view of APPNP, consider $J(U)=\tfrac{\alpha}{2}\|U-Q\|_F^2+\tfrac{1-\alpha}{2}\operatorname{tr}\bigl(U^{\top}(I-S)U\bigr)$,
whose first term anchors the state near $Q$ and whose second term is a
degree-normalised Dirichlet energy, so that $\alpha$ balances closeness to the original prediction against smoothness over the graph. Since $\nabla J(U)=\alpha(U-Q)+(1-\alpha)(I-S)U$, a gradient step with unit step size gives $U-\nabla J(U)=\alpha Q+(1-\alpha)SU=\mathcal T(U)$: one propagation step is exactly one unit-step gradient step on $J$ \citep{zhu2021interpretingunifyinggraphneural}.

Sharpening descends the Gini term of a single node. For a propagated class distribution $p_i$, define
\begin{equation}
E_\eta(u;p_i)
=
\operatorname{KL}(u\|p_i)
+\frac{\eta}{2}\left(1-\|u\|^2\right),
\qquad u\in\Delta^{C-1}
\label{eq:egp}
\end{equation}
where the KL term keeps the update close to $p_i$ and the Gini term favours more decisive class distributions. Stationarity gives the fixed-point condition $u=\softmax(\log p_i+\eta u)$, and one fixed-point step initialized at $u=p_i$, as in the limited fixed-point iterations of PottsMGNet \citep{tai2024pottsmgnet}, returns exactly the sharpening map \eqref{eq:sharpen-simplex}. \Cref{app:full_derivation} gives the full derivation, and that single step already decreases the energy it comes from.

\begin{proposition}[Sharpening descends the local energy]
\label{prop:descent}
Let $p\in\Delta^{C-1}$ with $p>0$, let $\eta\ge0$ and let $p^{*}=r_\eta(p)$.
Then $E_\eta(p^{*};p)\le E_\eta(p;p)$, and consequently
$\tfrac{\eta}{2}\bigl(\|p^{*}\|^{2}-\|p\|^{2}\bigr)\ge\operatorname{KL}(p^{*}\|p)\ge0$,
so one sharpening step does not increase the Gini term. Moreover, $p^{*}$ preserves the ordering of the classes of $p$.
\end{proposition}

The proof (\cref{app:sharpening}) uses the concavity of the Gini term: sharpening minimizes a linearized upper bound of $E_\eta(\cdot\,;p)$ that is tight at $p$, so iterating it is a concave--convex procedure for \eqref{eq:egp} \citep{concave_convex}.

\paragraph{Restart is not enough} Sharpening also changes the behaviour of deep propagation. Restart does not by itself protect a community from a confident neighbouring one: on a two-community graph, APPNP and PPR-Prob with $\alpha=0.1$ misclassify one entire community for every depth $K\ge3$, while PtS with $\eta=4$ keeps every node correct at every depth (\cref{prop:two-community} and \cref{fig:class-preservation} in \cref{app:restart}). The mechanism there is that propagation cannot push a node's correct-class probability below $0.546$ while it is at least $0.6$ everywhere, and one sharpening step lifts $0.546$ back above $0.6$.

\begin{takeaway}
The Potts decomposition motivates combining agreement between neighbours with decisiveness within each node. PtS propagates class probabilities, then applies one local KL--Gini sharpening step while preserving the propagated row mass. Sharpening preserves the current predicted class, so it affects labels only through subsequent propagation. Setting $\eta=0$ recovers PPR-Prob exactly.
\end{takeaway}

\section{Experimental protocol}
\label{sec:protocol}

We use nine homophilic graphs: WikiCS \citep{mernyei2022wikicswikipediabasedbenchmarkgraph}, Cora-TAPE, PubMed-TAPE and TAPE-Arxiv23 \citep{he2024harnessing}, ogbn-arxiv and ogbn-products \citep{hu2020ogb}, and Ele-Photo, Ele-Computers and Books-History \citep{yan2023a}. We use Roman-Empire and Amazon-Ratings \citep{platonov2024criticallookevaluationgnns} as heterophilic controls and report them separately. WikiCS uses its first ten supplied training/validation splits and its shared test mask, the OGB graphs use their official split, and all other graphs use ten unstratified random 60/20/20 splits. We use three backbone seeds per split, each (graph, split, backbone seed) defines one \emph{unit}. Dataset statistics are given in \cref{app:datasets}.

For each unit we train a graph-blind MLP on clean features, minimizing cross-entropy on the training labels, and freeze the checkpoint with the highest clean validation accuracy, architecture and optimizer settings are given in \cref{app:model_sizes}. The GCN~\citep{kipf2017gcn} and GraphSAGE~\citep{hamilton2017graphsage} backbones follow the same training and corruption protocol, model sizes are given in \cref{tab:model_sizes}.

At inference, we corrupt all nodes as
\begin{equation*}
X_{ij}^{(\sigma,r)}
=X_{ij}+\sigma s_j\xi_{ij}^{(r)},
\qquad \xi_{ij}^{(r)}\sim\mathcal N(0,1),
\end{equation*}
where $s_j$ is the standard deviation of feature $j$ over the training nodes and $\sigma\in\{0,0.5,1,1.5,2\}$. We draw three independent Gaussian fields $\xi^{(r)}$ per graph and split, shared across backbone seeds and methods, and reuse them across severities, so that the severities form a nested noise ladder. Cora-TAPE and PubMed-TAPE use RoBERTa-base~\citep{liu2019robertarobustlyoptimizedbert} CLS embeddings computed without fine-tuning, and the three CS-TAG graphs (Ele-Photo, Ele-Computers, Books-History) use the published RoBERTa-base CLS embeddings. For each corruption, we compute $Z=f_\theta(X^{(\sigma,r)})$ and $Q=\operatorname{softmax}(Z)$ once and use these across methods and experiments. APPNP refers to the propagation stage of \citet{gasteiger2019appnp}, applied post hoc to frozen logits $Z$, rather than the full model trained end to end.

APPNP, PPR-Prob and PtS are each tuned independently with 250 Optuna trials~\citep{akiba2019optuna} per unit, severity and noise draw, maximizing validation accuracy. APPNP and PPR-Prob tune $(\alpha,K)$, and PtS additionally tunes $\eta$, including $\eta=0$, full search spaces are in \cref{app:search} and the selected values in \cref{app:hyperparameters}. Unless stated otherwise, hyperparameters are selected at the same severity at which they are evaluated, transfer across severities is reported in \cref{fig:transfer}. For each severity, we average test accuracies over draws, then over seeds, within each split, and report mean $\pm$ standard deviation across splits. For the OGB graphs, which have a single official split, standard deviations instead describe variation across
seed$\times$draw repeats. Differences between methods are computed within (split, seed, draw) before aggregation.

We compare against two external post-hoc baselines using the same frozen
predictions, graph and validation data. LAME-Graph is the graph-adapted LAME of \citet{boudiaf2022lame}, which replaces the feature affinity by $S$ (\cref{app:std}) and tunes the interaction strength and the number of iterations using
250 Optuna trials. Graph-TV applies the non-local total-variation objective of \citet{yang2025regularizing} post hoc to the frozen predictions, with $\varepsilon=1$ and the regularization strength selected by validation accuracy over a grid of 26 candidates. Graph-TV is evaluated at $\sigma\in\{0,2\}$ only and excludes ogbn-products because of its memory requirements. Both baselines are detailed in \cref{app:external_baselines}.

\section{Results}

\subsection{Accuracy on clean and corrupted inputs}
Across the nine main graphs with a frozen MLP, PtS improves mean test accuracy over independently tuned PPR-Prob by $0.57$ percentage points on clean inputs and $1.09$ at $\sigma=2$. The corresponding gains over post-hoc APPNP
are $1.71$ and $3.90$ points (\cref{tab:repair}).
With the frozen MLP backbone at $\sigma=2$, mean accuracy rises from $41.95\%$ for the unrefined predictions $Q$ to $56.34\%$ for APPNP and $60.24\%$ for PtS.

\begin{table}[H]
\centering\small
\caption{Test accuracy (\%) with a frozen MLP backbone. Each method is tuned independently at the evaluation severity. The mean excludes the heterophilic controls.}
\label{tab:repair}
\resizebox{\linewidth}{!}{%
\begin{tabular}{@{}lrrrrrrrr@{}}
\toprule
& \multicolumn{4}{c}{$\sigma=0$} & \multicolumn{4}{c}{$\sigma=2$} \\
\cmidrule(lr){2-5}\cmidrule(lr){6-9}
Dataset & $Q$(MLP) & APPNP & PtS & PtS $-$ APPNP & $Q$(MLP) & APPNP & PtS & PtS $-$ APPNP \\
\midrule
WikiCS & 72.67 & 77.78 & \textbf{78.42} & +0.64 $\pm$ 0.50 & 46.36 & 67.97 & \textbf{72.15} & +4.18 $\pm$ 1.75 \\
Cora-TAPE & 60.51 & 76.16 & \textbf{79.74} & +3.58 $\pm$ 1.23 & 43.82 & 68.16 & \textbf{72.30} & +4.13 $\pm$ 1.99 \\
PubMed-TAPE & 81.59 & 84.35 & \textbf{85.18} & +0.83 $\pm$ 0.64 & 64.56 & 78.87 & \textbf{80.73} & +1.86 $\pm$ 1.44 \\
TAPE-Arxiv23 & 67.59 & 69.53 & \textbf{69.72} & +0.19 $\pm$ 0.10 & 28.56 & 35.63 & \textbf{38.91} & +3.29 $\pm$ 0.26 \\
ogbn-arxiv & 55.95 & 66.10 & \textbf{66.56} & +0.46 $\pm$ 0.51 & 22.65 & 33.86 & \textbf{40.62} & +6.76 $\pm$ 1.73 \\
ogbn-products & 59.58 & 70.58 & \textbf{72.17} & +1.58 $\pm$ 0.24 & 22.46 & \textbf{28.54} & 27.33 & -1.21 $\pm$ 0.17 \\
Ele-Photo & 66.83 & 71.87 & \textbf{74.91} & +3.04 $\pm$ 0.63 & 44.15 & 55.97 & \textbf{61.43} & +5.46 $\pm$ 1.04 \\
Ele-Computers & 60.96 & 75.37 & \textbf{80.11} & +4.74 $\pm$ 0.96 & 38.05 & 59.94 & \textbf{69.56} & +9.63 $\pm$ 1.00 \\
Books-History & 81.38 & 82.59 & \textbf{82.88} & +0.29 $\pm$ 0.23 & 66.91 & 78.11 & \textbf{79.13} & +1.02 $\pm$ 0.61 \\
\addlinespace[2pt]\multicolumn{9}{@{}l}{\scriptsize\textit{heterophilic controls}} \\
Roman-Empire & 65.54 & \textbf{65.65} & 65.54 & -0.11 $\pm$ 0.13 & \textbf{19.34} & 19.30 & 19.30 & -0.01 $\pm$ 0.02 \\
Amazon-Ratings & 49.52 & 51.63 & \textbf{52.74} & +1.12 $\pm$ 0.35 & 34.17 & 36.74 & \textbf{36.78} & +0.04 $\pm$ 0.08 \\
\midrule
\textbf{Mean (9 main)} & 67.45 & 74.93 & \textbf{76.63} & +1.71 & 41.95 & 56.34 & \textbf{60.24} & +3.90 \\
\bottomrule
\end{tabular}%
}
\end{table}

PtS beats APPNP on all nine main graphs on clean inputs and on eight out of nine at $\sigma=2$, where the positive gains range from $1.02$ to $9.63$ points. The exception is ogbn-products, where APPNP is better on noisy data while PtS beats it on clean data.
The advantage over APPNP is negligible on the two heterophilic controls at $\sigma =2$.
For the other baselines, at $\sigma=2$, PtS also improves mean accuracy over LAME-Graph by $3.75$ points across nine graphs. LAME-Graph uses the coupled Potts update, placing the graph interaction inside the softmax while retaining $Q$ as its reference, PtS propagates first and sharpens relative to the propagated distribution (\cref{app:std}). Against Graph-TV, the improvement is $4.24$ points on the eight graphs evaluated by both methods (\cref{tab:external_baselines}).

\subsection{Sharpening and propagation depth}


At $\sigma =2$, independently tuned PPR-Prob improves mean accuracy over APPNP by $2.81$ points, and PtS improves over PPR-Prob by an additional $1.09$ points. On clean inputs, the PtS--PPR-Prob gap is $0.57$ points. Per-dataset comparisons are given in \cref{tab:decomp}, and the full corruption ladder is reported in \cref{tab:full-gaussian-ladder}.
Additionally, to test if the selected PtS configurations depend on sharpening, we keep their ($\alpha, K$) fixed and set $\eta =0$ without retuning (Reaction OFF). 
Mean accuracy at $\sigma =2$ then falls from $60.24\%$ to $54.93\%$, a drop of $5.31$ points (\cref{tab:decomp_compressed}). \Cref{fig:noise} visualizes how the two contributions change across the full corruption ladder, and \cref{tab:full-gaussian-ladder} gives the dataset-level values.

\begin{table}[H]
\centering
\caption{Mean accuracy (\%) over the nine main graphs at $\sigma=2$ with the frozen MLP. PPR-Prob is independently tuned, Reaction OFF reuses the PtS-selected propagation settings without retuning.}
\label{tab:decomp_compressed}
\small
\begin{tabular}{@{}lll@{}}
\toprule
Method  & Sharpening & Accuracy \\
\midrule
PPR-Prob  &$\eta=0$ & 59.15 \\
PtS & Selected $\eta$ & \textbf{60.24} \\
Reaction OFF  & $\eta=0$ & 54.93 \\
\bottomrule
\end{tabular}
\end{table}


Next, we vary the propagation depth $K$, while holding each curve's restart $\alpha$ and sharpening strength $\eta$ fixed (\cref{fig:depth}).  Logit-Sharp applies the sharpening step on top of logit-space propagation using APPNP.

Without restart, increasing $K$ from $2$ to $100$ reduces clean accuracy by $33.8$ points for APPNP and $31.2$ for PPR-Prob, compared with $2.2$ points for PtS with $\eta =200$. With noise $\sigma =2$, APPNP and PPR-Prob lose $21.1$ and $20.5$ points, while PtS with $\eta= 200$ reaches $57.7\%$ at both depths with a maximum in between.
With restart $\alpha=0.1$ on clean inputs, accuracy decreases from $74.5\%$ to $70.4\%$ for APPNP and from $75.7\%$ to $74.8\%$ for PtS with $\eta=200$ between $K=2$ and $K=100$: losses of approximately $4.1$ and $0.9$ percentage points.
At $K=100$, both PtS curves remain above APPNP and PPR-Prob in every restart and corruption condition shown; complete depth results are reported in \cref{tab:depth-values}.

\begin{figure}[H]
    \centering
    \includegraphics[width=0.85\linewidth]{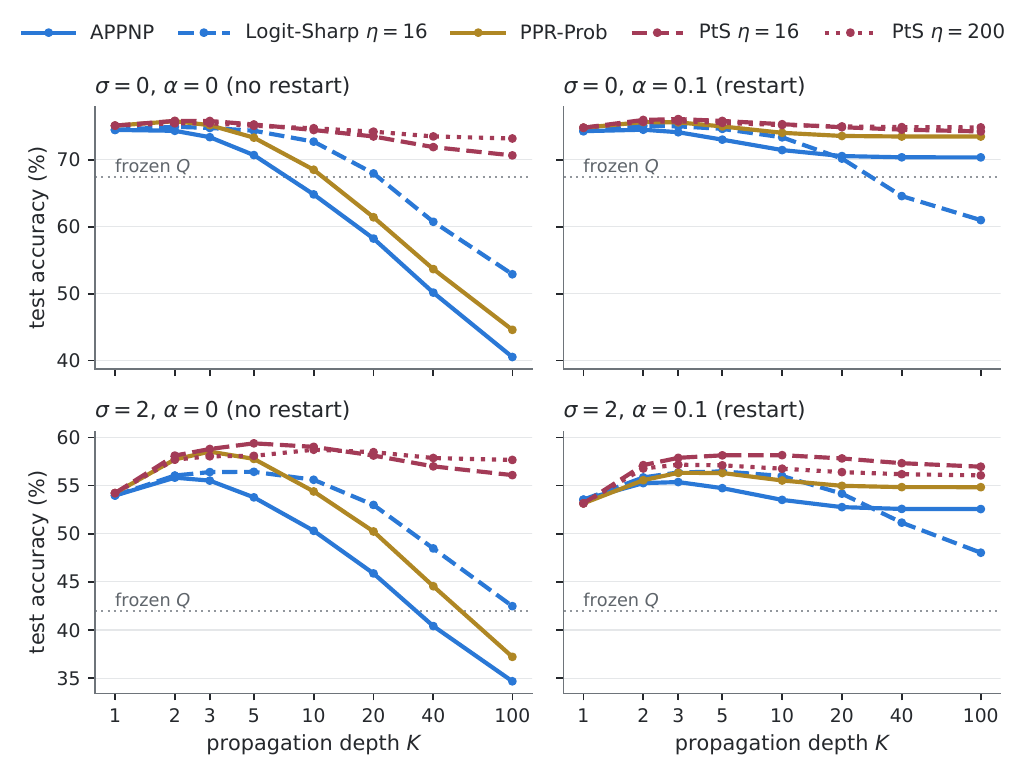}
    \caption{Mean accuracy over the nine main graphs at fixed $(\alpha,\eta)$. Columns vary the restart, rows vary the corruption severity.}
    \label{fig:depth}
\end{figure}

\subsection{Corruption severity}
\Cref{fig:noise} shows positive mean gains for PtS over both APPNP and PPR-Prob at every tested corruption severity. With hyperparameters selected independently for each method at the evaluation severity, the gain over APPNP is $1.71$ percentage points on clean inputs and $3.90$ at $\sigma=2$, while the gains over PPR-Prob are $0.57$ and $1.09$ points. 
The main comparisons select hyperparameters at the severity used for evaluation. \Cref{fig:transfer} examines how improvements transfer across selection and evaluation severities. The mean PtS-APPNP gain remains positive across all 25 combinations. When both methods are selected using clean validation data, PtS keeps a $2.28$ point advantage at $\sigma=2$, compared with $3.90$ points under severity-matched selection. 

\begin{figure}[H]
    \centering
    \begin{minipage}[t]{0.50\linewidth}
        \vspace{0pt}
        \centering
        \includegraphics[width=\linewidth]{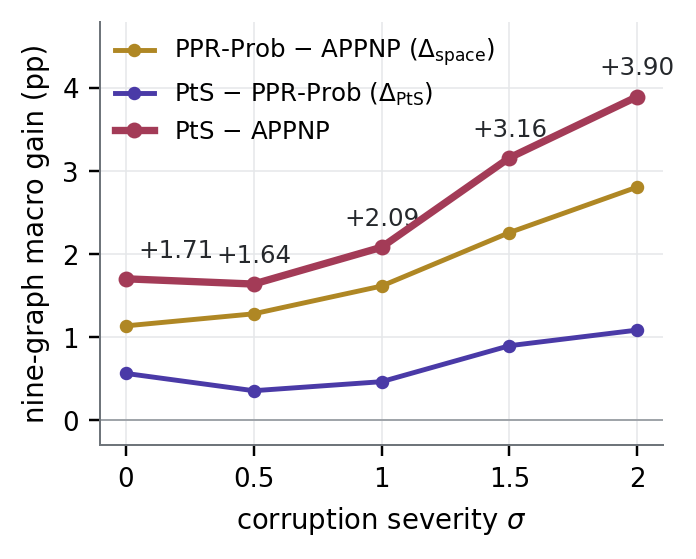}
        \caption{Mean test-accuracy gains in percentage points over the nine main homophilic graphs with a frozen MLP. The curves compare PPR-Prob with APPNP, PtS with PPR-Prob, and PtS with APPNP. Each method is independently selected at the evaluation severity.}
        \label{fig:noise}
    \end{minipage}
    \hfill
    \begin{minipage}[t]{0.48\linewidth}
        \vspace{0pt}
        \centering
        \includegraphics[width=\linewidth]{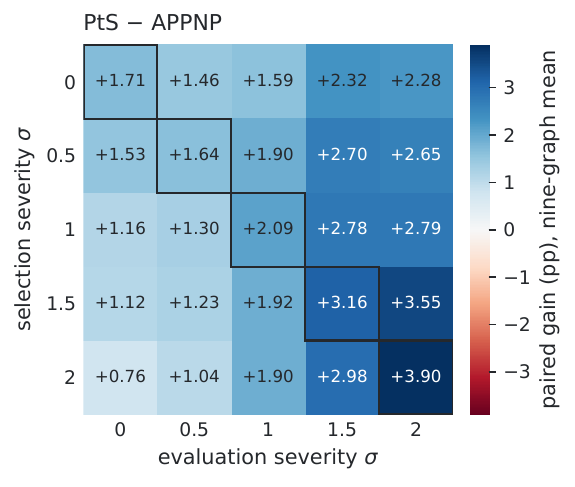}
        \caption{Mean PtS$-$APPNP test-accuracy gains in percentage points over the nine homophilic graphs with a frozen MLP. Rows give the severity used for hyperparameter selection and columns the evaluation severity. The first row shows clean-selection transfer.}
        \label{fig:transfer}
    \end{minipage}
\end{figure}

\subsection{Backbone}
Under the same-severity selection protocol used for the main comparisons, we apply PtS to frozen GCN and GraphSAGE predictions, this still gives positive gains of $0.88$ and $1.23$ points at $\sigma =2$ (\cref{tab:backbone-summary}). Both backbones show positive gains over APPNP on all nine main graphs at $\sigma=2$, and on seven (GCN) and nine (GraphSAGE) on clean inputs (\cref{app:stronger_backbones}). Their smaller gains compared to MLPs are consistent with graph-aware backbones profiting less from post-hoc graph propagation. 
The Correct \& Smooth extension is evaluated separately in \cref{app:correct_smooth}.

\begin{table}[H]
\centering\small
\caption{Mean test-accuracy gain of PtS over independently tuned
APPNP across the nine main datasets, in percentage points.}
\label{tab:backbone-summary}
\begin{tabular}{@{}lrr@{}}
\toprule
Frozen backbone & $\sigma=0$ & $\sigma=2$ \\
\midrule
MLP       & +1.71 & +3.90 \\
GCN       & +0.41 & +0.88 \\
GraphSAGE & +0.47 & +1.23 \\
\bottomrule
\end{tabular}
\end{table}

\subsection{Calibration and runtime}
With an MLP at $\sigma =2$, PtS has a higher raw negative log-likelihood (NLL) and expected calibration error (ECE) than APPNP, so its uncalibrated confidences are less trustworthy. Temperature scaling fitted on validation nodes \citep{guo2017calibration} reduces the mean NLL of PtS below that of APPNP without changing accuracy, although its ECE stays higher (\cref{calibration}).
PtS costs between $1.3$ and $54$ ms per propagation step (\cref{tab:runtime}).

\begin{takeaway}
PtS improves mean accuracy on clean and corrupted inputs and substantially reduces the accuracy loss from deep propagation.
It keeps its advantage under clean-only selection and with graph-aware backbones. Sharpening worsens raw calibration, which temperature scaling largely mitigates.
\end{takeaway}

\section{Discussion}

\paragraph{Findings}
Sharpening never changes a node's own predicted class, so its accuracy gain comes from subsequent propagation steps.
Turning the sharpening off at PtS-selected $(\alpha,K)$ lowers accuracy by $5.31$ points, showing that these selected propagation settings depend on sharpening (\cref{tab:decomp}).
Gains over APPNP are generally larger in higher local-homophily bins and in lower prediction-confidence quintiles (\cref{tab:per_node_performance}).
The largest gain is in the $0.6$--$0.8$ local-homophily bin rather than at $0.8$--$1$. One possible explanation is that, when neighbours are confident in their predictions, propagation alone already produces a decisive and mostly correct distribution, so sharpening has less to add, while in more uncertain local regimes, averaging can leave the node with a correct but small margin and one sharpening step restores it (\cref{lem:restore}) before the next propagation. The two local-homophily bins below 0.4 have negative PtS-APPNP gains.
Post-hoc refinement also recovers a substantial part of the corruption damage: with the MLP backbone the frozen prediction loses $25.5$ points between $\sigma =0$ and $\sigma =2$, and PtS cuts that loss to $16.4$ points, a $36\%$ reduction against $27\%$ for APPNP (\cref{tab:repair}).
The graph-aware backbones lose less to begin with ($8$ and $14$ points) and refinement still recovers $25\%$ and $22\%$ of it (\cref{app:stronger_backbones}). Their smaller headroom might be due to their propagation operators acting as a form of Laplacian smoothing that reduces Dirichlet disagreement \citep{li2018deeperinsightsgraphconvolutional}.

\paragraph{Limitations}
The main results select hyperparameters at the severity at which they are evaluated, so they assume validation nodes that reflect test-time conditions. \cref{fig:transfer} shows that the gain survives clean-only validation, but the gain shrinks.
We study Gaussian corruption of node features only, leaving missing features and perturbations of the graph itself to future work. Additionally, sharpening worsens raw calibration, but this is mitigated by temperature scaling (\cref{calibration}).
All nodes share a single $\eta$, and propagation favours agreement between connected nodes so that neither component can adapt to the reliability of a local neighbourhood, and misleading neighbours can reinforce errors rather than correct them.
The negligible gains on the heterophilic controls and the loss against APPNP on ogbn-products (\cref{tab:repair}) show that the benefit is not universal.
Node- or edge-adaptive parameters are therefore a natural extension, ACMP provides a related example of learned attractive and repulsive interactions \citep{wang2025acmpallencahnmessagepassing}.
Finally, PtS does not minimize the full Potts energy on the graph. Propagation is normalised PPR, and we use the propagated distribution as the reference for the local sharpening step, so \cref{prop:descent} is only a local descent guarantee at a single node. 
Empirically, the two PtS steps act on complementary parts of the energy. In the example on WikiCS, propagation reduces the Dirichlet term, while sharpening increases it but strongly reduces the Gini term (\cref{fig:energy_evolve}). After a few iterations, both energy terms stabilize and validation accuracy reaches a plateau, while the propagation without sharpening continues to deteriorate.
The energy formulation also provides a natural route for extending PtS with additional regularizers or constraints. Variational and Potts-based image segmentation methods have incorporated geometric and topological priors such as convexity, connectivity and volume constraints \citep{liu2020convexshapepriordeep,liu2022std,Enforcing_connectivity}. Similar constraints could be carried over to graphs that carry geometric or topological information that can be exploited.

\paragraph{Conclusion}
Propagation over the graph can recover accuracy from degraded features, but its standard safeguard against oversmoothing, restart, reinjects the predictions that have become unreliable.
Splitting the relaxed Potts interaction into  Dirichlet and Gini terms adds an extra safeguard: propagate class probabilities, then sharpen each node.
PtS adds one parameter to APPNP-style propagation, needs only the frozen class distribution and the graph  and costs a few milliseconds per propagation step.
With an MLP backbone, it improves mean test accuracy over independently tuned APPNP by $3.90$ points at $\sigma=2$ and $1.71$ on clean inputs across nine graphs, and removes most of the accuracy drop of deep propagation. Gains are smaller with graph-aware backbones and negligible on heterophilic graphs.
Smoothing across neighbours and decisiveness within nodes can be treated as separate operations, and keeping nodes decisive enables deeper propagation.

\subsection*{AI use statement}

Generative AI tools assisted with editing and polishing the text, including checking the consistency of notation and formula definitions and identifying awkward phrasing and spelling errors. AI tools also assisted with literature search, including identifying LAME as potentially relevant prior work, the main author read and evaluated the original paper before incorporating it into the manuscript and experimental comparison. AI tools were also used for code review and for cleaning up the submitted code for readability, including removing stale experiments not included in the paper, numerical consistency checks, and generating Figures 1 and 5. They were also consulted for feedback on experiments and as discussion partners for exploratory ideas. All AI-assisted material was reviewed and checked for accuracy by the authors, who take responsibility for the final text, code, claims, and results of this work.

\subsection*{Reproducibility statement}
The refinement method is fully specified by \eqref{eq:PtS-update}--\eqref{eq:reaction_with_m},
and derived in \cref{app:full_derivation}; the proofs of \cref{prop:descent} and
\cref{prop:two-community} are given in \cref{app:sharpening} and \cref{app:restart}.
The experimental protocol, including the backbones, corruption model, splits, and
aggregation, is described in \cref{sec:protocol}; dataset statistics are listed in
\cref{tab:datasets}, model sizes in \cref{app:model_sizes}, search spaces in
\cref{tab:search}, and the selected hyperparameters in \cref{app:hyperparameters}.
Baseline objectives and search settings are given in \cref{app:external_baselines}.
Code is available at:  
\ifanonymous
    \url{https://anonymous.4open.science/r/Predict-then-propagate-04C0/}.
\else
    \url{https://github.com/Prebenno/Propagate-Then-Sharpen}.
    \subsection*{Funding}
    P.J.B. is supported by a PhD fellowship funded by the Research Council of Norway through the STIPINST programme (project no. 342607), hosted by NORCE.
\fi

\bibliography{iclr2027_conference}
\bibliographystyle{iclr2027_conference}
\newpage
\FloatBarrier
\appendix

\section*{Appendix overview}

This appendix provides the theoretical derivations, experimental details, and additional analyses supporting the main paper.

\paragraph{Theory and derivations (\crefrange{app:potts_energy}{app:sharpening}).}
We derive the graph energy underlying PtS, its decomposition into Dirichlet and Gini terms, and the categorical sharpening step. We then connect the construction to Potts and LAME and analyze properties of sharpening and repeated propagation, including class-order preservation, local Gini descent, propagation collapse, and the effect of sharpening at increasing depth.

\paragraph{Experimental setup (\crefrange{app:datasets}{app:hyperparameters}).}
We provide full details on the datasets, frozen base models, hyperparameter search spaces, and selected hyperparameters used throughout the experiments.

\paragraph{Additional comparisons (\crefrange{app:external_baselines}{app:ladder}).}
We compare PtS with the external post-hoc baselines LAME-Graph and Graph-TV, evaluate stronger GCN and GraphSAGE backbones, study how the sharpening step can be incorporated into Correct \& Smooth, and report results across the full Gaussian corruption ladder.

\paragraph{Ablations and mechanism (\crefrange{app:mass}{app:per_node}).}
We examine row-mass preservation, propagation depth, the evolution of the Dirichlet and Gini components, performance across nodes with different local homophily, degree, and prediction confidence

\paragraph{Calibration and computational cost (\crefrange{app:runtime}{app:calibration}).}
Finally, we report runtime and computational overhead and evaluate calibration before and after temperature scaling.

%
%
\section{From the continuous Potts energy to the graph energy}
We start from an entropy-regularized Potts formulation and transfer it from a continuous domain to a graph. The derivations show how predictions from a frozen classifier become a KL term, and how the Potts interaction decomposes into a Dirichlet and Gini term. This decomposition motivates combining propagation between nodes with local sharpening of class distributions.

\label[appendix]{app:potts_energy}
Let
\begin{equation*}
    \Delta^{C-1}
    =
    \left\{
    u\in\mathbb{R}^C:
    u_c\geq 0,\;
    \sum_{c=1}^{C}u_c=1
    \right\}
\end{equation*}
denote the probability simplex. In the multiphase formulation, an
assignment field
$v(x)\in\Delta^{C-1}$ associates each position $x \in \Omega$ with a soft class distribution.

Our motivation comes from the entropy-regularized multiphase
formulation in Soft Threshold Dynamics \citep{liu2022std}
and the related two-phase formulation in PottsMGNet
\citep{tai2024pottsmgnet}. With the perimeter-scaling constants
absorbed into $\lambda$, we write the continuous energy as
\begin{equation}
\label{APPNDX:pottsmgnet}
\begin{aligned}
\min_{v:\Omega\to\Delta^{C-1}}\Biggl[
&\int_{\Omega} v(x)^{\top}g(x)\,dx
+\epsilon\int_{\Omega}\sum_{c=1}^{C}
v_c(x)\ln v_c(x)\,dx
\\
&+\frac{\lambda}{2}
\int_{\Omega}\int_{\Omega}
G_{\rho}(x-y)\left(1-v(x)^{\top}v(y)\right)dy\,dx
\Biggr]
\end{aligned}
\end{equation}
Here $\epsilon >0$, $\lambda \geq 0$, and $G_\rho$ is a Gaussian kernel. The last term is an approximation of the Potts boundary penalty.
The first term assigns a cost $g(x)$ to each class at each position, and
the entropy term regularizes the assignments within the simplex, with the
convention $0\log0=0$.

\subsection{Graph analogue}
We replace spatial positions with graph nodes, integrals with
sums, and the Gaussian kernel with a symmetric nonnegative affinity matrix.
\begin{equation*}
    W\in\mathbb{R}^{N\times N}
    \qquad
    W=W^\top
    \qquad
    W_{ij}\geq0.
\end{equation*}

The graph analogue of \cref{APPNDX:pottsmgnet} is
\begin{equation}
\label{APPNDX:Graph_analogue}
\min_{\{u_i\in\Delta^{C-1}\}_{i=1}^{N}}
\left[
\sum_{i}u_i^{\top}g_i
+
\epsilon
\sum_{i,c}
u_{ic}\ln u_{ic}
+
\frac{\lambda}{2}
\sum_{i,j}
W_{ij}
\left(
1-u_i^{\top}u_j
\right)
\right].
\end{equation}
For one-hot assignments $u_i=e_{y_i}$ \[
1-u_i^{\top}u_j
=
\mathbf{1}[y_i\neq y_j].
\]
and the interaction reduces to a weighted Potts regularization that penalizes disagreement between connected nodes.

We obtain the assignment cost from the frozen classifier.
Let $q_i\in\Delta^{C-1}$ denote row $i$ of $Q$, the predicted class distribution at node $i$,
with $q_{ic}>0$ for every class, and choose
$g_{ic}=-\epsilon\log q_{ic}$.

Then the first two terms in \eqref{APPNDX:Graph_analogue} become
\begin{align*}
\sum_i u_i^{\top}g_i
+
\epsilon\sum_{i,c}u_{ic}\log u_{ic}
&=
\epsilon
\sum_{i,c}
u_{ic}
\log\frac{u_{ic}}{q_{ic}}
\\
&=
\epsilon
\sum_i
\operatorname{KL}(u_i\|q_i).
\end{align*}

After dividing by $\epsilon$ and defining
$\gamma=\lambda/\epsilon$, we get
\begin{equation}
\label{eq:graph-potts-kl}
E_{\gamma}(U;Q)
=
\sum_i
\operatorname{KL}(u_i\|q_i)
+
\frac{\gamma}{2}
\sum_{i,j}
W_{ij}
\left(
1-u_i^{\top}u_j
\right).
\end{equation}
The energy therefore combines anchoring to the frozen classifier's predictions with a relaxed Potts interaction between the nodes.

\paragraph{Gini + Dirichlet}
For a symmetric W, we define $d_i = \sum_j W_{ij}$. 
Using the polarization identity
\[
\|u_i-u_j\|^2
=
\|u_i\|^2+\|u_j\|^2-2u_i^\top u_j
\]
and the symmetry of $W$, we get
\begin{equation*}
\frac{1}{2}
\sum_{i,j}
W_{ij}
\left(
1-u_i^\top u_j
\right)
=
\underbrace{
\frac{1}{4}
\sum_{i,j}
W_{ij}
\|u_i-u_j\|^2
}_{E_{\mathrm D}(U)}
+
\underbrace{
\frac{1}{2}
\sum_i
d_i
\left(
1-\|u_i\|^2
\right)
}_{E_{\mathrm G}(U)}.
\end{equation*}

The Dirichlet term $E_{\mathrm D}$ penalizes differences between the class distributions of connected nodes, and is zero when the class distributions are identical on all connected nodes.
The Gini term $E_{\mathrm G}$ penalizes distributions that spread mass over several classes, and is zero when all node distributions are one-hot.

The two terms describe different properties. Two neighbouring nodes can carry identical but uncertain class distributions: their connecting edge is then not penalized by $E_{\mathrm D}$, while both nodes still contribute to $E_{\mathrm G}$.

\section{Derivation of the sharpening step}
\label[appendix]{app:full_derivation}
Here, we derive the sharpening step from a local energy that combines the KL distance to the propagated distribution with a Gini term. The derivation shows why one fixed-point step yields the reaction used in PtS. When $\eta = 0$, the reaction becomes an identity.
The decomposition above gives
\[
E_\gamma(U;Q)
=
\sum_i\operatorname{KL}(u_i\|q_i)
+\gamma E_{\mathrm D}(U)+\gamma E_{\mathrm G}(U).
\]
which motivates treating graph smoothing and local categorical sharpening as separate operations.

PtS uses the probability-space propagation defined in \cref{TRANSPORT},
\begin{equation*}
U^{(0)}=Q,
\qquad
\widetilde U^{(k+1)}
=
\alpha Q+(1-\alpha)SU^{(k)}.
\label{eq:appnp}
\end{equation*}

For a positive propagated row $h_i=\widetilde U_i^{(k+1)}$, let
$m_i=\mathbf 1^\top h_i$ and $p_i=h_i/m_i$. Since the rows of $S$ do not necessarily sum to one, this normalization is what ensures $p_i \in\Delta^{C-1}$.

We hold $p$ fixed and use the assignment cost $g_c = -\log p_c$. Combining this cost with the entropy regularization gives

\begin{equation*}
E(u;p)
=
\sum_cu_c\log\frac{u_c}{p_c}
+\frac\eta2\bigl(1-\|u\|^2\bigr),
\qquad u\in\Delta^{C-1}.
\label{eq:local-reaction-energy}
\end{equation*}
The first term keeps the update close to the propagated distribution $p$ and is zero when $u=p$, while the second term penalizes indecision and is zero when $u$ is one-hot.

We use one global sharpening strength $\eta\geq0$ for all nodes.
This replaces the degree-dependent weighting in the Gini term.
To handle the constraint $\sum_cu_c=1$, we introduce
a Lagrange multiplier $\nu$:
\begin{equation*}
L(u,\nu;p)
=
E(u;p)+\nu\left(\sum_cu_c-1\right).
\end{equation*}
Stationarity in the simplex gives
\begin{equation*}
\frac{\partial\mathcal{L}}{\partial u_c}
=
\log\frac{u_c}{p_c}+1-\eta u_c+\nu
=
0.
\end{equation*}
With $p$ held fixed, the corresponding constrained gradient flow is
\begin{equation*}
    \dot u_c(t)
    =
    -\log\frac{u_c(t)}{p_c}
    -1+\eta u_c(t)-\nu(t),
    \qquad u(0)=p,
\end{equation*}
where $\dot u_c(t)$ denotes the time derivative of $u_c(t)$.
The multiplier $\nu(t)$ is chosen so that
$\sum_{c=1}^{C}\dot u_c(t)=0$.
Since $\sum_{c=1}^{C}p_c=1$, this preserves
$\sum_{c=1}^{C}u_c(t)=1$.

An implicit Euler step of size $\tau>0$ gives
\begin{equation*}
\frac{u_c-p_c}{\tau}
+\log\frac{u_c}{p_c}+1-\eta u_c+\nu=0.
\label{eq:local-implicit-euler}
\end{equation*}
Since the new unknown distribution $u$ appears both in the time-step term and in the gradient, we use a fixed-point iteration similar to the one in
PottsMGNet \citep[Section~5.2]{tai2024pottsmgnet}.
Starting from $v^{(0)}=p$, we evaluate the non-logarithmic terms at
$v^{k}$ and solve for the next iterate:
\begin{equation*}
v^{(k+1)}
=
\softmax\!\left(
\log p+\eta v^{(k)}-\frac{v^{(k)}-p}{\tau}
\right).
\label{eq:local-implicit-picard}
\end{equation*}
We use one inner iteration. Since $v^{(0)}=p$, the last term is zero
and the update becomes
\begin{equation*}
p^{*}_c
=
\frac{p_c\exp(\eta p_c)}{\sum_b p_b\exp(\eta p_b)},
\qquad
p^{*}=\softmax(\log p+\eta p).
\label{eq:categorical-reaction}
\end{equation*}
Because the fixed-point iteration is initialized at $v^{(0)}=p$,
the time-step term vanishes in the first iteration, so the step size $\tau$ does
not appear in the implemented reaction. We use this single iteration rather than solving to convergence, and restore the row mass afterwards:
\[
R_\eta(h_i)
=
m_i\softmax\!\left(\log p_i+\eta p_i\right),
\qquad
\mathbf{1}^\top R_\eta(h_i)=m_i .
\]
The restored mass determines the row's total contribution in the next propagation step. With $\eta=0$ we get $R_0(h_i)=h_i$, so that PtS reduces to PPR-Prob with the same $(\alpha,K)$. The final class distributions are obtained by row-normalizing
$U^{(K)}$.

\section{Relation to Potts and LAME}
\label[appendix]{app:std}
This part shows the connection between the graph energy, the coupled fixed-point update, and LAME-Graph. In the coupled update, the neighbourhood contribution is in the softmax, while the original predictions are kept as the reference. PtS instead uses separate steps for propagation and sharpening, with the propagated distribution as the reference in the local sharpening step. 

For a symmetric $W$, every stationary point of the energy in \cref{eq:graph-potts-kl} must satisfy
\begin{equation}
\label{eq:potts-fixed-point}
u_i
=
\operatorname{softmax}
\left(
\log q_i+\gamma(WU)_i
\right).
\end{equation}
Here the neighbourhood interaction sits inside the softmax, while $q_i$ is kept as the reference. PtS instead propagates first and then uses the propagated distribution $p_i$ as the anchor for the reaction. The Potts decomposition motivates this construction, but it does not imply that the PtS iteration reduces the global energy. We evaluate the coupled iteration \eqref{eq:potts-fixed-point} with $B=\gamma S$ under the same protocol as everything else, reported as LAME-Graph in \cref{tab:external_baselines,tab:baselines-clean}.

\paragraph{Relationship to LAME}

Once the constant term $\frac{\gamma}{2}\sum_{i,j}W_{ij}$ is removed from \cref{eq:graph-potts-kl}, the energy has the same KL-plus-affinity form as the objective of LAME \citep{boudiaf2022lame}, whose update can be written as
\begin{equation*}
u^{(t+1)}_i = \operatorname{softmax}\bigl(\log q_i + (Bu^{(t)})_i\bigr).
\end{equation*}
In the original LAME formulation the affinity matrix $B$ is constructed from pretrained feature representations, using a $k$-nearest-neighbour affinity as well as linear and Gaussian kernels. Setting $B =\gamma W$ recovers the iteration in \cref{eq:potts-fixed-point}, which connects the coupled update to LAME with a graph-derived affinity; this is the LAME-Graph baseline. PtS uses a different update: propagation combines the current state with the original predictions, and a separate mass-preserving reaction then sharpens the propagated distribution $p_i$.

\newpage

\section{Properties of sharpening}
Here we examine both the properties of a single sharpening step and its behaviour under repeated propagation. We show that sharpening keeps the class ordering and decreases local Gini. The propagation analysis shows why the nodes' predictions can collapse without restart, while the example in \cref{fig:class-preservation} shows that even restart is not always sufficient to keep the correct predictions correct. 

\label[appendix]{app:sharpening}

\subsection{Pure propagation collapses}
\label[appendix]{app:collapse}

Assume $\alpha =\eta=0$ and that the graph is connected and undirected, with a self-loop at every node. Let $Q > 0$ have row sums equal to one. Propagation without restart and without sharpening is then
\begin{equation}
    \label{Propagation_no_a_S}
    U^{(K)}=S^KQ .
\end{equation}
From \citet[Proposition~1]{oono2020oversmoothing}, $S$ has a single eigenvalue $1$ while all other eigenvalues have absolute value less than 1, the corresponding unit eigenvector $\phi\propto\widetilde D^{1/2}\mathbf 1$ is strictly positive. Write $v=\phi$, with $v_j$ denoting its entry at node $j$.

Since $S$ is real and symmetric, we can decompose it into its eigenvectors, where $N$ is the number of nodes and $\lambda_1 =1$,
\begin{equation*}
    S^K=\lambda_1^K \phi\phi^\top
    +\sum_{r=2}^{N}\lambda_r^K \phi_r\phi_r^\top.
\end{equation*}
Let $K \to \infty$. Since $|\lambda_r| < 1$ for $r\ge2$ then $\lambda_r^K \to 0$ leaving us with 
\begin{equation}
    \label{Propagation_limit}
    S^K \to \phi\phi^\top
\end{equation}

Let $p_i^{(K)}$ be the final normalised prediction at node i.
\begin{equation*}
\begin{aligned}
    p_i^{(K)}
    &=\frac{U_i^{(K)}}{\sum_c U_{ic}^{(K)}}
    &&\textit{By definition}\\
    &=\frac{(S^KQ)_i}{\sum_c(S^KQ)_{ic}}
    &&\textit{By \eqref{Propagation_no_a_S}}\\
    &\to\frac{(\phi\phi^\top Q)_i}{\sum_c(\phi\phi^\top Q)_{ic}}
    &&\textit{By \eqref{Propagation_limit}}\\
    &=\frac{\phi_i\sum_jv_jQ_j}
    {\phi_i\sum_jv_j\underbrace{\sum_cQ_{jc}}_{=1}}
\end{aligned}
\end{equation*}

Which result in 
\begin{equation*}
    \lim_{K\to\infty}p_i^{(K)}
    =\frac{\sum_jv_jQ_j}{\sum_jv_j}.
\end{equation*}
The limit does not depend on $i$: it is a weighted average of the initial predictions, so every node converges to the same normalised class distribution. Pure probability propagation therefore collapses, as observed empirically in \cref{fig:depth}.

The same argument applies to APPNP at $\alpha=0$, where the propagated state is the logit matrix $Z$ rather than $Q$. Here $S^{K}Z\to vv^{\top}Z$, so node $i$ receives the logit row $\phi_i\,c$ with the common vector $c=\sum_j \phi_jZ_j$. Since $\phi_i>0$, the softmax of $\phi_i c$ has the same maximizing class at every node: the predicted distributions need not coincide, but the predicted \emph{labels} all collapse onto a single class.

\paragraph{With $\alpha \in (0,1]$}
With restart, the iteration converges to the PPNP solution instead \citep{gasteiger2019appnp},
\begin{equation*}
    U^{\infty} = \alpha(I_n-(1-\alpha)S)^{-1}Q ,
\end{equation*}
which depends on $i$ through $Q$, so restart prevents collapse. It does so by reinjecting $Q$ at every step, which is what makes the quality of $Q$ decisive, \cref{prop:two-community} shows that this is not always enough.

\subsection{Restart alone does not preserve a minority community}
\label[appendix]{app:restart}

\begin{proposition}
\label{prop:two-community}
Let $G$ consist of two $9$-node cliques $\mathcal A$ and $\mathcal B$ joined by a perfect matching, so that with added self-loops every node draws $90\%$ of its propagation weight from its own community. Let the true class be $1$ on $\mathcal A$ and $2$ on $\mathcal B$, with frozen predictions $Q_i=(0.9,0.1)$ for $i\in\mathcal A$ and $Q_j=(0.4,0.6)$ for $j\in\mathcal B$, and let $\alpha=0.1$. Then APPNP and PPR-Prob misclassify every node of $\mathcal B$ for every depth $K\ge3$, while PtS with $\eta=4$ classifies every node correctly at every depth.
\end{proposition}

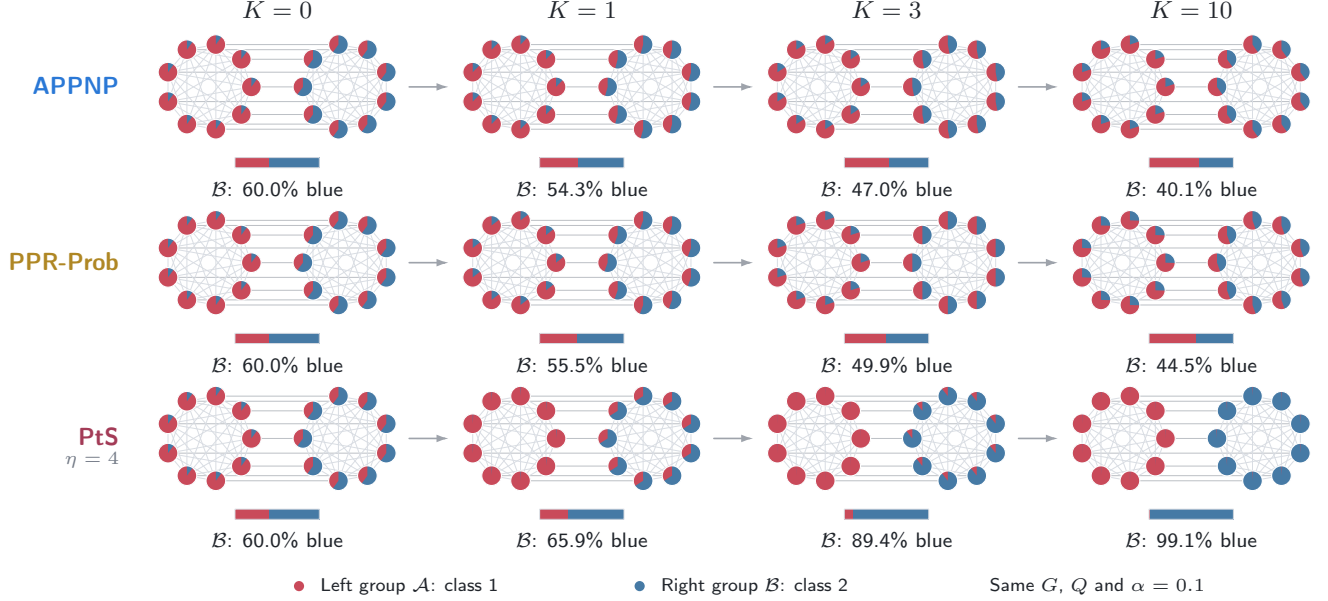
\begin{figure}[H]
    \centering
\begingroup
\definecolor{ink}{HTML}{20242A}
\definecolor{muted}{HTML}{747B87}
\definecolor{hair}{HTML}{D9DDE3}
\definecolor{redc}{HTML}{C94A5A}
\definecolor{bluec}{HTML}{467AA5}

\definecolor{methodAPPNP}{HTML}{2A78D6}
\definecolor{methodPPR}{HTML}{AF8724}
\definecolor{methodPtS}{HTML}{A33B57}

\tikzset{
  ptsdepth/edge/.style={draw=hair,line width=.32pt},
  ptsdepth/bridge/.style={draw=muted!48,line width=.32pt},
  ptsdepth/arrow/.style={-{Latex[length=1.6mm,width=1.0mm]},
    draw=muted!70,line width=.6pt},
  ptsdepth/stage/.style={font=\sffamily\bfseries\fontsize{9}{10.5}\selectfont,
    text=ink,align=center},
  ptsdepth/tiny/.style={font=\sffamily\fontsize{6.7}{8}\selectfont,
    text=muted,align=center},
  ptsdepth/number/.style={font=\sffamily\fontsize{7.4}{8.7}\selectfont,
    text=ink,align=center}
}

\newcommand{\ptsdepthpie}[1]{%
  \def\radius{.126}
  \pgfmathsetmacro{\endangle}{90+360*(#1)}
  \fill[bluec] (0,0) circle (\radius);
  \fill[redc] (0,0)--(90:\radius)
    arc[start angle=90,end angle=\endangle,radius=\radius]--cycle;
  \draw[white,line width=.42pt] (0,0) circle (\radius);
}

\newcommand{\ptsdepthgraph}[4]{%
  \begin{scope}[shift={(#1,#2)}]
    \foreach \i in {1,...,9}{
      \pgfmathsetmacro{\ang}{360*(\i-1)/9}
      \coordinate (a\i) at ({-.88+.55*cos(\ang)},{.55*sin(\ang)});
      \coordinate (b\i) at ({ .88-.55*cos(\ang)},{.55*sin(\ang)});
    }
    \foreach \i in {1,...,9}{
      \draw[ptsdepth/bridge] (a\i)--(b\i);
    }
    \foreach \i in {1,...,9}{
      \foreach \j in {1,...,9}{
        \ifnum\i<\j\relax
          \draw[ptsdepth/edge] (a\i)--(a\j);
          \draw[ptsdepth/edge] (b\i)--(b\j);
        \fi
      }
    }
    \foreach \i in {1,...,9}{
      \begin{scope}[shift={(a\i)}]\ptsdepthpie{#3}\end{scope}
      \begin{scope}[shift={(b\i)}]\ptsdepthpie{#4}\end{scope}
    }
  \end{scope}
}

\newcommand{\ptsdepthpill}[4]{%
  \begin{scope}[shift={(#1,#2)}]
    \pgfmathsetmacro{\rw}{1.08*(#3)}
    \fill[redc] (-.54,-.065) rectangle ++(\rw,.13);
    \fill[bluec] ({-.54+\rw},-.065) rectangle (.54,.065);
    \draw[hair,line width=.35pt] (-.54,-.065) rectangle (.54,.065);
    \node[ptsdepth/number,anchor=north] at (0,-.105)
      {$\mathcal B$: #4\% blue};
  \end{scope}
}

\noindent\resizebox{\linewidth}{!}{%
\begin{tikzpicture}[x=1cm,y=1cm,line cap=round,line join=round]
  \path[use as bounding box] (-3.75,-6.64) rectangle (13.27,1.14);
  \foreach \x/\k in {0/0,3.9/1,7.8/3,11.7/10}{
    \node[ptsdepth/stage] at (\x,.99) {$K=\k$};
  }
  \node[ptsdepth/stage,anchor=east, text = methodAPPNP] at (-1.90,.02) {APPNP};
  \node[ptsdepth/stage,anchor=east, text =methodPPR] at (-1.90,-2.23) {PPR-Prob};
  \node[ptsdepth/stage,anchor=east,text=methodPtS] at (-1.90,-4.48) {PtS};
  \node[ptsdepth/tiny,anchor=east] at (-1.90,-4.78) {$\eta=4$};
  \foreach \y in {0,-2.25,-4.5}{
    \foreach \x in {0,3.9,7.8}{
      \draw[ptsdepth/arrow] ({\x+1.69},\y)--({\x+2.20},\y);
    }
  }

  \ptsdepthgraph{0}{0}{0.900000000000}{0.400000000000}
  \ptsdepthpill{0}{-0.97}{0.400000000000}{60.0}
  \ptsdepthgraph{3.9}{0}{0.876855365639}{0.457298514411}
  \ptsdepthpill{3.9}{-0.97}{0.457298514411}{54.3}
  \ptsdepthgraph{7.8}{0}{0.841961700866}{0.529680649862}
  \ptsdepthpill{7.8}{-0.97}{0.529680649862}{47.0}
  \ptsdepthgraph{11.7}{0}{0.800905994411}{0.598637923299}
  \ptsdepthpill{11.7}{-0.97}{0.598637923299}{40.1}

  \ptsdepthgraph{0}{-2.25}{0.900000000000}{0.400000000000}
  \ptsdepthpill{0}{-3.22}{0.400000000000}{60.0}
  \ptsdepthgraph{3.9}{-2.25}{0.855000000000}{0.445000000000}
  \ptsdepthpill{3.9}{-3.22}{0.445000000000}{55.5}
  \ptsdepthgraph{7.8}{-2.25}{0.799272000000}{0.500728000000}
  \ptsdepthpill{7.8}{-3.22}{0.500728000000}{49.9}
  \ptsdepthgraph{11.7}{-2.25}{0.745302706461}{0.554697293539}
  \ptsdepthpill{11.7}{-3.22}{0.554697293539}{44.5}

  \ptsdepthgraph{0}{-4.5}{0.900000000000}{0.400000000000}
  \ptsdepthpill{0}{-5.47}{0.400000000000}{60.0}
  \ptsdepthgraph{3.9}{-4.5}{0.990188767793}{0.340538853274}
  \ptsdepthpill{3.9}{-5.47}{0.340538853274}{65.9}
  \ptsdepthgraph{7.8}{-4.5}{0.996951040970}{0.106423863299}
  \ptsdepthpill{7.8}{-5.47}{0.106423863299}{89.4}
  \ptsdepthgraph{11.7}{-4.5}{0.995219425858}{0.008554730508}
  \ptsdepthpill{11.7}{-5.47}{0.008554730508}{99.1}

  \fill[redc] (.28,-6.39) circle (.065);
  \node[ptsdepth/tiny,anchor=west,text=ink] at (.43,-6.39)
    {Left group $\mathcal A$: class 1};
  \fill[bluec] (4.64,-6.39) circle (.065);
  \node[ptsdepth/tiny,anchor=west,text=ink] at (4.79,-6.39)
    {Right group $\mathcal B$: class 2};
  \node[ptsdepth/tiny,anchor=west,text=ink] at (9.00,-6.39)
    {Same $G$, $Q$ and $\alpha=0.1$};
\end{tikzpicture}%
}
\endgroup
    \caption{
        Class probabilities with increasing depth.
        All methods use the same graph, initial predictions,
        and $\alpha=0.1$, PtS uses $\eta=4$.
        The numbers give the correct-class probability
        in group $\mathcal B$.
        Self-loops are included in the updates but omitted
        from the drawing.
    }
    \label{fig:class-preservation}
\end{figure}

Each node is joined to one distinct node in the other group, so after adding self-loops every node has ten connections: nine within its community (eight clique neighbours and itself) and one across. All degrees are equal, so $\widetilde D=10I$ and $S=\widetilde A/10$ is both symmetric and row-stochastic, with weight $0.9$ inside the community and $0.1$ across it. Propagation therefore preserves unit row sums, and so does the reaction, so all three iterations stay on the simplex. By symmetry, all nodes of a community carry the same state at every depth, and it suffices to track one node per community.

\paragraph{PPR-Prob}
Let $a_k$ and $b_k$ be the class-1 probabilities in $\mathcal A$ and $\mathcal B$. The update \eqref{eq:ppr-prob} gives
\[
a_{k+1}=0.09+0.9\,(0.9a_k+0.1b_k),
\qquad
b_{k+1}=0.04+0.9\,(0.1a_k+0.9b_k).
\]
The sum $\Sigma_k=a_k+b_k$ satisfies $\Sigma_{k+1}=0.13+0.9\Sigma_k$ with $\Sigma_0=1.3$, hence $\Sigma_k=1.3$ for all $k$. The difference $\Delta_k=a_k-b_k$ satisfies $\Delta_{k+1}=0.05+0.72\,\Delta_k$ with $\Delta_0=0.5$, hence
\[
\Delta_k=\tfrac{5}{28}+\left(0.5-\tfrac{5}{28}\right)0.72^{\,k},
\]
which decreases monotonically to $5/28\approx0.179$. A node of $\mathcal B$ is misclassified exactly when $b_k=(1.3-\Delta_k)/2>1/2$, that is when $\Delta_k<0.3$. Since $\Delta_2=0.3452$ and $\Delta_3=0.29854$, the nodes of $\mathcal B$ are classified correctly for $K\le2$ and misclassified for every $K\ge3$.

\paragraph{APPNP}
With two classes, the softmax depends only on the logit difference, so let $x_k$ and $y_k$ be the class-1 minus class-2 logit differences in $\mathcal A$ and $\mathcal B$, starting from $x_0=\log 9$ and $y_0=\log\frac23$. APPNP applies the same affine recursion to the logits, so $\Sigma_k=x_k+y_k$ satisfies $\Sigma_{k+1}=0.1\Sigma_0+0.9\Sigma_k$ and stays at $\Sigma_0=\log 6$, while $\Delta_k=x_k-y_k$ satisfies $\Delta_{k+1}=0.1\Delta_0+0.72\Delta_k$ with $\Delta_0=\log 13.5$ and decreases monotonically to $0.1\Delta_0/0.28\approx0.930$. A node of $\mathcal B$ is misclassified exactly when $y_k=(\log 6-\Delta_k)/2>0$, that is when $\Delta_k<\log 6\approx1.7918$. Since $\Delta_2\approx1.7969$ and $\Delta_3\approx1.5541$, the nodes of $\mathcal B$ are again classified correctly for $K\le2$ and misclassified for every $K\ge3$.

\paragraph{PtS}
We claim that for $\eta=4$ every node keeps at least probability $0.6$ on its correct class at every depth. This holds at $k=0$, where the correct-class probabilities are $0.9$ on $\mathcal A$ and $0.6$ on $\mathcal B$. Assume it holds at step $k$ and consider a node $i$ with correct class $y$. Since the rows are distributions, the propagated value satisfies
\[
\widetilde U^{(k+1)}_{iy}
=0.1\,Q_{iy}+0.9\sum_j S_{ij}U^{(k)}_{jy}
\ \ge\
0.1\cdot 0.6+0.9\,(0.9\cdot0.6+0.1\cdot 0)
=0.546 ,
\]
using $Q_{iy}\ge0.6$, the induction hypothesis inside the community, and non-negativity across it. In the two-class case the sharpening map satisfies
\[
\frac{r_\eta(p)_y}{1-r_\eta(p)_y}
=\frac{p_y}{1-p_y}\,e^{\eta(2p_y-1)},
\]
which is increasing in $p_y$, so $\widetilde U^{(k+1)}_{iy}\ge0.546$ implies
\[
U^{(k+1)}_{iy}\ \ge\
\frac{0.546\,e^{4(0.546)}}
     {0.546\,e^{4(0.546)}+0.454\,e^{4(0.454)}}
=0.6347\ldots>0.6 .
\]
The invariant is restored, so by induction every node keeps more than probability $0.6$ on its correct class, and hence is classified correctly, at every depth. \qed

\subsection{Sharpening preserves the class ordering}
\label[appendix]{app:ordering}
From the update, we define
\begin{equation}
    \label{pc_written_out}
    p^*_c = \frac{p_ce^{\eta p_c}}{\sum_bp_be^{\eta p_b}} ,
\end{equation}
where the entries of $p^*$ are positive and sum to one. For any two classes $c$ and $b$,
\begin{equation}
    \frac{p^*_c}{p^*_b} =\frac{p_c}{p_b}\,e^{\eta (p_c-p_b)} .
    \label{eq:ratio}
\end{equation}
Since $\eta\ge0$, the ratio is at least $p_c/p_b$ whenever $p_c\ge p_b$, so sharpening reinforces a node's existing class preference and in particular leaves its predicted class unchanged.

The same identity quantifies how much confidence a single step can restore after propagation has eroded it.

\begin{lemma}[Recovering a confidence level]
\label{lem:restore}
Let $p\in\Delta^{C-1}$ with $p>0$ have a unique leading class $c$ and margin
$\mu=p_c-\max_{b\neq c}p_b>0$. Then
\[
p^{*}_c\ \ge\ \frac{1}{1+(C-1)e^{-\eta\mu}} ,
\]
so for any target confidence $\beta\in(0,1)$ we have $p^{*}_c\ge \beta$ as soon as
$\eta\ \ge\ \mu^{-1}\log\!\bigl((C-1)\beta/(1-\beta)\bigr)$.
In particular $p^{*}_c\to1$ as $\eta\to\infty$.
\end{lemma}

\begin{proof}
For $b\neq c$, \eqref{eq:ratio} gives $p^{*}_b/p^{*}_c=(p_b/p_c)e^{-\eta(p_c-p_b)}\le e^{-\eta\mu}$, since $p_b\le p_c$ and $p_c-p_b\ge\mu$. Summing over the $C-1$ competing classes,
\[
\frac{1}{p^{*}_c}=1+\sum_{b\neq c}\frac{p^{*}_b}{p^{*}_c}\le 1+(C-1)e^{-\eta\mu} ,
\]
which is the stated bound. It is at least $\beta$ precisely when $(C-1)e^{-\eta\mu}\le(1-\beta)/\beta$, that is when $\eta\mu\ge\log\bigl((C-1)\beta/(1-\beta)\bigr)$.
\end{proof}

A node whose leading class survives propagation with margin $\mu$ can therefore be returned to any confidence level by a large enough $\eta$, which is the mechanism behind \cref{prop:two-community}. The bound also shows the cost: the same $\eta$ is applied to every node, including nodes whose leading class is wrong.

\subsection{Proof of \cref{prop:descent}}
\label[appendix]{app:descent}

Fix $p\in\Delta^{C-1}$ with $p>0$ and $\eta\ge0$, and recall the local energy \eqref{eq:egp},
\[
E_\eta(u;p)=\operatorname{KL}(u\|p)+\frac{\eta}{2}\bigl(1-\|u\|^2\bigr),
\qquad u\in\Delta^{C-1}.
\]
The KL term is convex in $u$ and the Gini term is concave, so $E_\eta$ is a difference of convex functions. Concavity of $u\mapsto-\frac{\eta}{2}\|u\|^2$ gives the linear upper bound
\[
-\frac{\eta}{2}\|u\|^2\ \le\ -\frac{\eta}{2}\|p\|^2-\eta\,p^{\top}(u-p) ,
\]
with equality at $u=p$. Adding $\operatorname{KL}(u\|p)+\frac{\eta}{2}$ to both sides defines a majorizer
\[
M(u)=\operatorname{KL}(u\|p)-\eta\,p^{\top}u+\text{const} ,
\qquad
E_\eta(u;p)\le M(u),
\qquad
E_\eta(p;p)=M(p).
\]
By the Gibbs variational principle, $M$ is minimized over the simplex at
$u=\softmax(\log p+\eta p)=r_\eta(p)=p^{*}$, which is exactly the sharpening step \eqref{eq:sharpen-simplex}. Therefore
\[
E_\eta(p^{*};p)\ \le\ M(p^{*})\ \le\ M(p)\ =\ E_\eta(p;p) ,
\]
which proves the descent claim. Writing it out, $\operatorname{KL}(p^{*}\|p)+\frac{\eta}{2}(1-\|p^{*}\|^2)\le\frac{\eta}{2}(1-\|p\|^2)$, that is
\[
\frac{\eta}{2}\bigl(\|p^{*}\|^{2}-\|p\|^{2}\bigr)\ \ge\ \operatorname{KL}(p^{*}\|p)\ \ge\ 0 ,
\]
so for $\eta>0$ the Gini impurity $1-\|p\|^2$ does not increase, and it strictly decreases unless $p^{*}=p$. The ordering claim is \cref{app:ordering}. The same majorization argument applied at an arbitrary iterate $v$ instead of $p$ shows that the full fixed-point iteration $v\mapsto\softmax(\log p+\eta v)$ is a concave--convex procedure, so every inner iteration, not only the first, decreases $E_\eta(\cdot\,;p)$. \qed

\newpage

\section{Datasets}
\label[appendix]{app:datasets}
\cref{tab:datasets} summarizes the nine graphs in the main board and the heterophilic control datasets. The datasets vary in nodes, edges, homophily, and classes. The control datasets are reported separately to show how the methods act when neighbouring nodes often belong to separate classes.
\begin{table}[H]
\centering
\caption{Dataset statistics. Edges are the reported undirected-pair counts, and $h$ denotes edge homophily.}
\label{tab:datasets}
\begin{tabular}{@{}lrrrrr@{}}
\toprule
Dataset & Nodes & Edges & Features & Classes & $h$ \\
\midrule
WikiCS & 11,701 & 215,603 & 300 & 10 & 0.65 \\
Cora-TAPE & 2,708 & 5,278 & 768 & 7 & 0.81 \\
PubMed-TAPE & 19,717 & 44,324 & 768 & 3 & 0.80 \\
TAPE-Arxiv23 & 46,198 & 38,863 & 300 & 40 & 0.64 \\
ogbn-arxiv & 169,343 & 1,157,799 & 128 & 40 & 0.65 \\
ogbn-products & 2,449,029 & 61,859,012 & 100 & 47 & 0.81 \\
Ele-Photo & 48,362 & 436,891 & 768 & 12 & 0.74 \\
Ele-Computers & 87,229 & 628,274 & 768 & 10 & 0.82 \\
Books-History & 41,551 & 251,590 & 768 & 12 & 0.64 \\
\midrule
Roman-Empire & 22,662 & 32,927 & 300 & 18 & 0.05 \\
Amazon-Ratings & 24,492 & 93,050 & 300 & 5 & 0.38 \\
\bottomrule
\end{tabular}
\end{table}

\section{Frozen backbones}
\label{app:backbones}
\Cref{tab:model_sizes} shows the architecture of the frozen base models that are used to produce the predictions. The MLP uses node features alone, while GCN and GraphSAGE use the graph structure as well. PtS is applied after these models have been trained and does not add any new trainable parameters to the base models. The MLP has architecture $d\rightarrow256\rightarrow256\rightarrow C$, with BatchNorm, ReLU, and dropout ($0.5$) after each hidden linear layer, and is trained with Adam (learning rate $0.01$, weight decay $5\times10^{-4}$) for at most 500 epochs with early-stopping patience 100.

\label[appendix]{app:model_sizes}
\begin{table}[H]
\centering\small
\caption{Number of parameters of each backbone, with the feature dimension $D$ and the
number of classes $C$. The last column gives the depth and width of the graph backbones.}
\label{tab:model_sizes}
\begin{tabular}{@{}lrrrrrl@{}}
\toprule
Dataset & $D$ & $C$ & MLP & GCN & GraphSAGE
& Conv.\ layers $\times$ width \\
\midrule
WikiCS         & 300 & 10 & 146,442 & 80,138  & 159,498 & $2\times256$ \\
Cora-TAPE      & 768 & 7  & 265,479 & 199,175 & 397,575 & $2\times256$ \\
PubMed-TAPE    & 768 & 3  & 264,451 & 198,147 & 395,523 & $2\times256$ \\
TAPE-Arxiv23   & 300 & 40 & 154,152 & 87,848  & 174,888 & $2\times256$ \\
ogbn-arxiv     & 128 & 40 & 110,120 & 110,120 & 218,664 & $3\times256$ \\
ogbn-products  & 100 & 47 & 104,751 & 36,015  & 71,215  & $3\times128$ \\
Ele-Photo      & 768 & 12 & 266,764 & 200,460 & 400,140 & $2\times256$ \\
Ele-Computers  & 768 & 10 & 266,250 & 199,946 & 399,114 & $2\times256$ \\
Books-History  & 768 & 12 & 266,764 & 200,460 & 400,140 & $2\times256$ \\
\midrule
Roman-Empire   & 300 & 18 & 148,498 & 82,194  & 163,602 & $2\times256$ \\
Amazon-Ratings & 300 & 5  & 145,157 & 78,853  & 156,933 & $2\times256$ \\
\bottomrule
\end{tabular}
\end{table}

\clearpage

\section{Hyperparameter search spaces}
\label[appendix]{app:search}

\Cref{tab:search} shows the search spaces and budget for the propagation methods and variants of correct and smooth. We select hyperparameters using validation accuracy, with a separate Optuna trial per method. APPNP and PPR-Prob choose the restart weight $\alpha$ and number of propagation steps $K$, while PtS also searches over sharpening strength $\eta$. The search space also includes $\eta = 0$, so validation can choose to turn off the reaction.

\begin{table}[H]
\centering
\caption{Search spaces and maximum numbers of hyperparameter evaluations per unit and severity. All methods use Optuna TPE except Graph-TV, which uses a grid search. The PtS search includes the exact sharpening off setting $\eta=0$ alongside the logarithmic range. For hyperparameter tuning on ogbn-products, the upper limits are reduced to $K=50$ and $T=50$.}
\label{tab:search}
\begin{tabular}{@{}llr@{}}
\toprule
Method & Search & Budget \\
\midrule
APPNP / PPR & $\alpha\in[0,1]$, $K\in\{1,\ldots,100\}$ & 250 \\
PtS & Same, $\eta=0$ or $\log_{10}\eta\in[-2,2.408]$ & 250 \\
LAME-Graph & $T\in\{1,\ldots,100\}$, $\gamma=0$ or $\log_{10}\gamma\in[-2,2.408]$ & 250 \\
Graph-TV & $\lambda=0$ or $\log_{10}\lambda\in\{-3,\ldots,3\}$ (25 log-spaced values), $\varepsilon=1$ fixed & $\le 26$ \\
C\&S & $\alpha_c,\alpha_s\in[0,1]$ & 250 \\
C\&S-PtS & Same, plus the PtS reaction search & 250 \\
\bottomrule
\end{tabular}
\end{table}

\section{Selected hyperparameters (MLP)}
\label[appendix]{app:hyperparameters}

\begin{table}[H]
\caption{Median selected values over the units of each (graph, severity), where
$K$ is the number of propagation steps, $\alpha$ the restart probability, $\eta$ the
sharpening strength, and ``off'' the percentage of units for which the search selected
the exact sharpening off setting $\eta=0$.
}
\label{tab:hyperparameters}
\centering
\begingroup
\small
\setlength{\tabcolsep}{4pt}
\renewcommand{\arraystretch}{0.95}
\setlength{\LTpre}{4pt}
\setlength{\LTpost}{0pt}
\begin{longtable}{@{}llrrrrrrr@{}}
\toprule
Dataset & $\sigma$
& $K_{\mathrm{APPNP}}$
& $K_{\mathrm{PPR}}$
& $K_{\mathrm{PtS}}$
& $\alpha_{\mathrm{APPNP}}$
& $\alpha_{\mathrm{PtS}}$
& $\eta_{\mathrm{PtS}}$
& off (\%) \\
\midrule
\endhead

WikiCS & 0 & 2 & 2 & 3 & 0.11 & 0.09 & 4.11 & 7 \\
 & 0.5 & 2 & 2 & 4 & 0.08 & 0.06 & 6.40 & 8 \\
 & 1 & 2 & 3 & 4 & 0.04 & 0.04 & 6.96 & 3 \\
 & 1.5 & 2 & 2 & 6 & 0.03 & 0.03 & 12.86 & 0 \\
 & 2 & 2 & 3 & 8 & 0.02 & 0.00 & 34.68 & 0 \\
\addlinespace[1pt]

Cora-TAPE & 0 & 3 & 6 & 54 & 0.06 & 0.13 & 2.24 & 0 \\
 & 0.5 & 3 & 6 & 38 & 0.05 & 0.12 & 1.78 & 6 \\
 & 1 & 4 & 6 & 34 & 0.03 & 0.10 & 1.42 & 7 \\
 & 1.5 & 4 & 7 & 34 & 0.03 & 0.08 & 1.35 & 4 \\
 & 2 & 5 & 6 & 38 & 0.02 & 0.06 & 1.46 & 7 \\
\addlinespace[1pt]

PubMed-TAPE & 0 & 30 & 28 & 46 & 0.32 & 0.31 & 1.51 & 20 \\
 & 0.5 & 26 & 24 & 52 & 0.24 & 0.24 & 1.33 & 30 \\
 & 1 & 4 & 7 & 15 & 0.11 & 0.15 & 0.82 & 21 \\
 & 1.5 & 4 & 7 & 22 & 0.05 & 0.13 & 1.13 & 10 \\
 & 2 & 5 & 8 & 24 & 0.03 & 0.10 & 1.61 & 6 \\
\addlinespace[1pt]

TAPE-Arxiv23 & 0 & 26 & 4 & 58 & 0.35 & 0.30 & 0.91 & 27 \\
 & 0.5 & 7 & 8 & 26 & 0.26 & 0.21 & 0.17 & 39 \\
 & 1 & 4 & 8 & 16 & 0.12 & 0.09 & 0.00 & 53 \\
 & 1.5 & 4 & 10 & 14 & 0.04 & 0.05 & 0.15 & 32 \\
 & 2 & 3 & 9 & 29 & 0.08 & 0.05 & 0.63 & 8 \\
\addlinespace[1pt]

ogbn-arxiv & 0 & 2 & 2 & 3 & 0.12 & 0.26 & 6.80 & 0 \\
 & 0.5 & 2 & 2 & 3 & 0.08 & 0.23 & 15.96 & 0 \\
 & 1 & 2 & 3 & 7 & 0.03 & 0.20 & 6.70 & 0 \\
 & 1.5 & 2 & 2 & 3 & 0.03 & 0.06 & 6.77 & 0 \\
 & 2 & 1 & 2 & 2 & 0.03 & 0.02 & 0.00 & 56 \\
\addlinespace[1pt]

ogbn-products & 0 & 2 & 2 & 17 & 0.13 & 0.36 & 3.34 & 0 \\
 & 0.5 & 2 & 2 & 4 & 0.04 & 0.31 & 42.74 & 0 \\
 & 1 & 2 & 1 & 1 & 0.41 & 0.27 & 0.00 & 56 \\
 & 1.5 & 2 & 2 & 2 & 0.48 & 0.33 & 0.00 & 78 \\
 & 2 & 2 & 2 & 1 & 0.47 & 0.29 & 0.00 & 67 \\
\addlinespace[1pt]

Ele-Photo & 0 & 1 & 1 & 1 & 0.02 & 0.03 & 1.70 & 37 \\
 & 0.5 & 1 & 1 & 1 & 0.02 & 0.07 & 6.59 & 22 \\
 & 1 & 1 & 1 & 2 & 0.01 & 0.03 & 7.80 & 22 \\
 & 1.5 & 1 & 2 & 2 & 0.01 & 0.00 & 9.37 & 19 \\
 & 2 & 1 & 2 & 3 & 0.02 & 0.00 & 8.73 & 13 \\
\addlinespace[1pt]

Ele-Computers & 0 & 2 & 3 & 10 & 0.01 & 0.24 & 6.17 & 0 \\
 & 0.5 & 2 & 3 & 7 & 0.00 & 0.10 & 7.27 & 0 \\
 & 1 & 2 & 3 & 8 & 0.00 & 0.03 & 11.10 & 0 \\
 & 1.5 & 3 & 3 & 9 & 0.00 & 0.01 & 17.22 & 1 \\
 & 2 & 3 & 4 & 12 & 0.00 & 0.00 & 23.25 & 0 \\
\addlinespace[1pt]

Books-History & 0 & 1 & 1 & 41 & 0.42 & 0.42 & 3.37 & 20 \\
 & 0.5 & 1 & 1 & 15 & 0.30 & 0.34 & 3.26 & 13 \\
 & 1 & 1 & 2 & 5 & 0.18 & 0.24 & 3.87 & 7 \\
 & 1.5 & 2 & 3 & 8 & 0.10 & 0.16 & 3.65 & 7 \\
 & 2 & 2 & 3 & 14 & 0.06 & 0.09 & 4.77 & 2 \\
\addlinespace[1pt]

Roman-Empire & 0 & 45 & 37 & 60 & 0.98 & 0.97 & 0.00 & 53 \\
 & 0.5 & 42 & 59 & 53 & 0.98 & 0.92 & 0.12 & 39 \\
 & 1 & 46 & 62 & 52 & 1.00 & 0.93 & 0.09 & 41 \\
 & 1.5 & 56 & 51 & 50 & 1.00 & 0.90 & 2.49 & 34 \\
 & 2 & 52 & 42 & 63 & 1.00 & 0.84 & 0.11 & 40 \\
\addlinespace[1pt]

Amazon-Ratings & 0 & 40 & 8 & 36 & 0.55 & 0.26 & 0.00 & 70 \\
 & 0.5 & 54 & 3 & 4 & 0.40 & 0.17 & 0.00 & 73 \\
 & 1 & 32 & 36 & 45 & 0.35 & 0.18 & 0.00 & 88 \\
 & 1.5 & 51 & 45 & 56 & 0.29 & 0.21 & 0.00 & 66 \\
 & 2 & 47 & 44 & 47 & 0.18 & 0.20 & 0.00 & 57 \\

\bottomrule
\end{longtable}
\addtocounter{table}{-1}
\endgroup
\end{table}

Across multiple graphs, hyperparameter tuning chose more propagation steps for PtS than for PPR-Prob (\cref{tab:hyperparameters}). The chosen sharpening strength varies across datasets and noise levels, and in some cases it is completely turned off. This suggests that sharpening is not always helpful. The depth also needs to be interpreted with the restart weight in mind, a large $K$ can have a limited interpretation when $\alpha$ is close to 1. We calculate medians separately for each hyperparameter.

%
%

\newpage

\section{External post-hoc baselines}
\label[appendix]{app:external_baselines}

Both baselines refine the same frozen predictions $Q$ on the same observed graph and evaluation units as PtS. Validation labels are used only for hyperparameter selection. The adaptations, graph operators and selection procedures are specified below.

\paragraph{LAME-Graph.}
LAME~\citep{boudiaf2022lame} refines a frozen classifier's output using a KL fidelity term and an agreement term weighted by affinities computed from pretrained feature representations. We replace these affinities by the self-looped graph operator $S$
and initialize $U^{(0)}=Q$. The update is
\[
U^{(t+1)}
=
\operatorname{RowSoftmax}
\bigl(\log Q+\gamma S U^{(t)}\bigr).
\]
The original predictions therefore remain the reference throughout the iteration, while the graph interaction enters inside the softmax. This is the coupled Potts update described in
\cref{app:std}.

We select the interaction strength $\gamma$ and a fixed iteration count $T$ by validation accuracy using 250 Optuna TPE trials per unit, severity and noise draw. The search uses $T\in\{1,\ldots,100\}$ and $\log_{10}\gamma\in[-2,2.408]$ for positive strengths, together with an explicit identity option that returns $Q$ exactly. We run the selected number of iterations rather than using the original energy-based stopping rule.

\paragraph{Graph-TV.}
We adapt the TV-regularized softmax of
\citet{yang2025regularizing} to frozen predictions by minimizing
\[
\sum_i \operatorname{KL}(u_i\|q_i)
+\lambda\,\mathrm{TV}_G(U),
\qquad u_i\in\Delta^{C-1}.
\]
The KL term keeps predictions close to $Q$, while TV penalizes disagreement across the graph. We use their classwise neighbourhood TV with loopless, degree-normalised graph weights, setting the input to $\log Q$ and $\varepsilon=1$.
Setting $\lambda=0$ returns $Q$.

We select $\lambda$ by validation accuracy from at most 26 grid candidates (\cref{tab:search}). Each dual solve uses at most 3000 iterations, and only candidates with relative primal--dual gap at most $10^{-3}$ are eligible for selection. Graph-TV is evaluated at $\sigma\in\{0,2\}$ on eight graphs, excluding ogbn-products because of memory requirements.

\begin{table}[H]
\centering
\caption{Comparison of post-hoc graph refinement methods at $\sigma = 2$, including the LAME-Graph and the Graph-TV methods.}
\label{tab:external_baselines}
\resizebox{\textwidth}{!}{%
\begin{tabular}{lrrrrrr}
\toprule
Dataset 
& $Q$ 
& APPNP 
& PPR-Prob 
& LAME-Graph 
& Graph-TV 
& PtS \\
\midrule
WikiCS & 46.36 & 67.97 & 68.98 & 67.80 & 65.10 & \textbf{72.15} \\
Cora-TAPE & 43.82 & 68.16 & 70.47 & 63.94 & 68.72 & \textbf{72.30} \\
PubMed-TAPE & 64.56 & 78.87 & 80.22 & 73.87 & 79.64 & \textbf{80.73} \\
TAPE-Arxiv23 & 28.56 & 35.63 & 38.26 & 35.26 & 36.54 & \textbf{38.91} \\
ogbn-arxiv & 22.65 & 33.86 & 40.43 & 40.38 & 35.25 & \textbf{40.62} \\
ogbn-products & 22.46 & \textbf{28.54} & 27.24 & 27.47 & -- & 27.33 \\
Ele-Photo & 44.15 & 55.97 & 61.00 & 59.58 & 57.06 & \textbf{61.43} \\
Ele-Computers & 38.05 & 59.94 & 66.85 & 62.45 & 59.64 & \textbf{69.56} \\
Books-History & 66.91 & 78.11 & 78.90 & 77.66 & 78.90 & \textbf{79.13} \\
\addlinespace[2pt]\multicolumn{7}{@{}l}{\scriptsize\textit{heterophilic controls}} \\
Roman-Empire & 19.34 & 19.30 & 19.30 & 19.29 & \textbf{19.31} & 19.30  \\
Amazon-Ratings & 34.17 & 36.74 & \textbf{36.80} & 36.73 & 36.73 & 36.78 \\
\midrule
 Mean (9 main) & {41.95} & {56.34} & {59.15} & {56.49} & {--} & {60.24} \\
{Mean (8 main, no products)} & {44.38} & {59.81} & {63.14} & {60.12} & {60.11} & {64.35}  \\
\bottomrule
\end{tabular}
}
\end{table}

\begin{table}[H]
\centering\small
\caption{External $Q{+}G$ baselines on clean features ($\sigma=0$), including the LAME-Graph and the Graph-TV methods.}
\label{tab:baselines-clean}
\resizebox{\linewidth}{!}{%
\begin{tabular}{@{}lrrrrrr@{}}
\toprule
Dataset
& $Q$
& APPNP
& PPR-Prob
& LAME-Graph
& Graph-TV
& PtS \\
\midrule
WikiCS & 72.67 & 77.78 & 78.24 & 78.18 & 78.18 & \textbf{78.42} \\
Cora-TAPE & 60.51 & 76.16 & 78.42 & 76.78 & 76.82 & \textbf{79.74} \\
PubMed-TAPE & 81.59 & 84.35 & 84.96 & 85.05 & 85.11 & \textbf{85.18}  \\
TAPE-Arxiv23 & 67.59 & 69.53 & 69.68 & \textbf{69.82} & 69.81 & 69.72  \\
ogbn-arxiv & 55.95 & 66.10 & 65.96 & \textbf{67.53} & 65.80 & 66.56 \\
ogbn-products & 59.58 & 70.58 & 71.32 & 71.40 & -- & \textbf{72.17} \\
Ele-Photo & 66.83 & 71.87 & 74.43 & 74.84 & 72.03 & \textbf{74.91}  \\
Ele-Computers & 60.96 & 75.37 & 78.71 & 78.32 & 76.47 & \textbf{80.11}\\
Books-History & 81.38 & 82.59 & 82.87 & 82.95 & \textbf{82.97} & 82.88  \\
\addlinespace[2pt]\multicolumn{7}{@{}l}{\scriptsize\textit{heterophilic controls}} \\
Roman-Empire & 65.54 & \textbf{65.65} & 65.54 & 65.55 & 65.57 & 65.54 \\
Amazon-Ratings & 49.52 & 51.63 & \textbf{52.77} & 52.05 & 52.36 & 52.74 \\
\midrule
{Mean (9 main)} & {67.45} & {74.93} & {76.06} & {76.10} & {--} & {76.63}\\
{Mean (8 main, no products)} & {68.43} & {75.47} & {76.66} & {76.68} & {75.90} & \textbf{77.19}  \\
\bottomrule
\end{tabular}
}
\end{table}

\Cref{tab:external_baselines,tab:baselines-clean} shows the results at $\sigma = 2$ and $\sigma =0$. The comparison with Graph-TV averages over 8 graphs without products, as it requires too much memory. PtS has a higher average accuracy in these comparisons, but is not best on all datasets. On clean data, LAME-Graph is better on the ogbn-arxiv, TAPE-Arxiv23, and Books-History datasets.

\section{Full decomposition}

\Cref{tab:decomp} gives the per-dataset values for the decomposition summarised in \cref{tab:decomp_compressed}, at $\sigma = 2$ with the frozen MLP
\label[appendix]{app:decomp_full}
\begin{table}[H]
\centering\small

\caption{Decomposition of PtS$-$APPNP at $\sigma=2$ into the change from
moving propagation into probability space, $\Delta_{\mathrm{space}}$
and the change from adding the reaction,
$\Delta_{\mathrm{PtS}}$. APPNP, PPR-Prob, and PtS are tuned independently. Reaction OFF sets $\eta=0$ at the $(\alpha,K)$ selected for PtS, without retuning}
\label{tab:decomp}
\begin{tabular}{@{}lrrrrrr@{}}
\toprule
Dataset & APPNP & PPR-Prob & PtS & Reaction OFF
& $\Delta_{\mathrm{space}}$ {\scriptsize PPR$-$APPNP}
& $\Delta_{\mathrm{PtS}}$ {\scriptsize PtS$-$PPR} \\
\midrule
WikiCS & 67.97 & 68.98 & 72.15 & 53.81 & +1.01 $\pm$ 1.59 & +3.17 $\pm$ 0.91 \\
Cora-TAPE & 68.16 & 70.47 & 72.30 & 64.83 & +2.31 $\pm$ 1.90 & +1.83 $\pm$ 0.65 \\
PubMed-TAPE & 78.87 & 80.22 & 80.73 & 76.80 & +1.35 $\pm$ 1.25 & +0.51 $\pm$ 0.27 \\
TAPE-Arxiv23 & 35.63 & 38.26 & 38.91 & 36.90 & +2.64 $\pm$ 0.28 & +0.65 $\pm$ 0.37 \\
ogbn-arxiv & 33.86 & 40.43 & 40.62 & 39.95 & +6.57 $\pm$ 1.71& +0.18 $\pm$ 0.22 \\
ogbn-products & 28.54 & 27.24 & 27.33 & 27.33 & -1.30 $\pm$ 0.22 & +0.09 $\pm$ 0.15 \\
Ele-Photo & 55.97 & 61.00 & 61.43 & 59.55 & +5.03 $\pm$ 0.75 & +0.42 $\pm$ 0.41 \\
Ele-Computers & 59.94 & 66.85 & 69.56 & 58.13 & +6.91 $\pm$ 0.84 & +2.72 $\pm$ 0.64 \\
Books-History & 78.11 & 78.90 & 79.13 & 77.11 & +0.79 $\pm$ 0.47 & +0.23 $\pm$ 0.19 \\
\midrule
\textbf{Mean (9 main)} & \textbf{56.34} & \textbf{59.15} & \textbf{60.24} & \textbf{54.93} & \textbf{+2.81} & \textbf{+1.09} \\
\bottomrule

\end{tabular}
\end{table}

Under independent validation selection and severe corruption $\sigma = 2$, moving propagation from logit space to probability space accounts for $2.81$ percentage points of the mean gain, and adding sharpening for a further $1.09$ (\cref{tab:decomp}). The split between the two varies by graph: on WikiCS it is $+1.01$ and $+3.17$, while on ogbn-arxiv it is $+6.57$ and $+0.18$. In the Reaction OFF column, we keep the $(\alpha, K)$ selected for PtS and set $\eta$ to zero without new tuning. Mean accuracy then falls to 54.93\%, below APPNP itself, so the propagation configurations selected for PtS are only viable in combination with sharpening: PtS generally selects deeper propagation than PPR-Prob (\cref{app:hyperparameters}). \Cref{fig:noise} shows how the mean gains change with corruption severity, with dataset-level results in \cref{tab:full-gaussian-ladder}. When hyperparameters are selected on clean data and evaluated at $\sigma=2$, PtS retains mean gains of $2.28$ percentage points over APPNP in the nine main graphs
(\cref{fig:transfer}).

\newpage

\section{Stronger backbones}
\label[appendix]{app:stronger_backbones}
We also examine whether PtS improves when frozen predictions come from base models that already use the graph structure. We apply the same method to predictions from GCN and GraphSAGE, with model sizes given in \cref{tab:model_sizes}. The comparison explores how much additional improvement we can get when the base model has already exploited neighbourhood information.

\subsection{GCN Backbone}

\begin{table}[H]
\centering
\caption{Accuracy (\%) with the frozen GCN backbone. $\Delta$ is PtS minus APPNP in percentage points. Bold marks the highest displayed accuracy in each condition.}
\label{tab:backbone-gcn}
\small
\setlength{\tabcolsep}{3pt}
\renewcommand{\arraystretch}{1.08}
\begin{tabular*}{\linewidth}{@{\extracolsep{\fill}}lrrrr@{\hspace{10pt}}rrrr@{}}
\toprule
& \multicolumn{4}{c}{Clean ($\sigma=0$)}
& \multicolumn{4}{c}{Noisy ($\sigma=2$)} \\
\cmidrule(lr){2-5}\cmidrule(lr){6-9}
Dataset & $Q(GCN)$ & APPNP & PtS & $\Delta$
        & $Q(GCN)$ & APPNP & PtS & $\Delta$ \\
\midrule
WikiCS & 79.08 & 79.05 & \textbf{79.40} & +0.35 $\pm$ 0.21 & 76.73 & 76.91 & \textbf{77.68} & +0.77 $\pm$ 0.30 \\
Cora-TAPE & 81.96 & 82.37 & \textbf{83.10} & +0.73 $\pm$ 0.49 & 76.20 & 78.34 & \textbf{79.94} & +1.60 $\pm$ 0.61 \\
PubMed-TAPE & 85.77 & 85.74 & \textbf{85.81} & +0.06 $\pm$ 0.05 & 80.41 & 82.27 & \textbf{82.73} & +0.46 $\pm$ 0.17 \\
TAPE-Arxiv23 & \textbf{69.28} & 69.26 & 69.26 & -0.00 $\pm$ 0.02 & 43.43 & 44.04 & \textbf{44.33} & +0.29 $\pm$ 0.11 \\
ogbn-arxiv & 71.83 & 71.90 & \textbf{72.48} & +0.58 $\pm$ 0.28 & 62.73 & 64.66 & \textbf{64.78} & +0.12 $\pm$ 0.33 \\
ogbn-products & 76.55 & 77.24 & \textbf{77.73} & +0.49 $\pm$ 0.17 & 63.81 & 69.16 & \textbf{71.01} & +1.86 $\pm$ 1.10 \\
Ele-Photo & 81.33 & 81.42 & \textbf{81.87} & +0.45 $\pm$ 0.18 & 78.47 & 78.94 & \textbf{79.67} & +0.73 $\pm$ 0.22 \\
Ele-Computers & 84.03 & 84.61 & \textbf{85.60} & +0.99 $\pm$ 0.14 & 77.19 & 80.58 & \textbf{82.49} & +1.91 $\pm$ 0.41 \\
Books-History & \textbf{82.68} & 82.66 & 82.66 & -0.00 $\pm$ 0.03 & 80.04 & 80.27 & \textbf{80.44} & +0.16 $\pm$ 0.12 \\
\midrule
Roman-Empire & \textbf{45.72} & 45.70 & 45.70 & +0.00 $\pm$ 0.02 & \textbf{20.38} & 20.36 & 20.35 & -0.01 $\pm$ 0.02 \\
Amazon-Ratings & \textbf{50.02} & 50.00 & 50.00 & -0.01 $\pm$ 0.05 & 37.22 & \textbf{37.93} & 37.82 & -0.11 $\pm$ 0.24 \\
\midrule
Mean (9 main) & {79.17} & {79.36} & {79.77} & {+0.41} & {71.00} & {72.80} & {73.67} & {+0.88} \\
\bottomrule
\end{tabular*}

\end{table}

With a GCN backbone, performance on clean data improves incrementally on many graphs, but PtS achieves higher accuracy than APPNP on all nine main graphs at $\sigma =2$ (\cref{tab:backbone-gcn}) and on seven out of nine graphs on clean graphs.

\subsection{GraphSAGE backbone}
\begin{table}[H]
\centering
\caption{Accuracy (\%) with the frozen GraphSAGE backbone. $\Delta$ is PtS minus APPNP in percentage points. Bold marks the highest displayed accuracy in each condition.}
\label{tab:backbone-graphsage}
\small
\setlength{\tabcolsep}{3pt}
\renewcommand{\arraystretch}{1.08}
\begin{tabular*}{\linewidth}{@{\extracolsep{\fill}}lrrrr@{\hspace{10pt}}rrrr@{}}
\toprule
& \multicolumn{4}{c}{Clean ($\sigma=0$)}
& \multicolumn{4}{c}{Noisy ($\sigma=2$)} \\
\cmidrule(lr){2-5}\cmidrule(lr){6-9}
Dataset & $Q(SAGE)$ & APPNP & PtS & $\Delta$
        & $Q(SAGE)$ & APPNP & PtS & $\Delta$ \\
\midrule
WikiCS & 78.72 & 79.08 & \textbf{79.41} & +0.33 $\pm$ 0.18 & 71.06 & 75.12 & \textbf{76.38} & +1.27 $\pm$ 0.43 \\
Cora-TAPE & 73.54 & 74.76 & \textbf{75.96} & +1.20 $\pm$ 1.17 & 68.18 & 70.89 & \textbf{72.69} & +1.80 $\pm$ 0.95 \\
PubMed-TAPE & 84.10 & 84.47 & \textbf{84.80} & +0.33 $\pm$ 0.18 & 72.26 & 78.63 & \textbf{80.04} & +1.41 $\pm$ 0.77 \\
TAPE-Arxiv23 & 70.04 & 70.20 & \textbf{70.35} & +0.15 $\pm$ 0.09 & 39.67 & 41.59 & \textbf{43.09} & +1.50 $\pm$ 0.44 \\
ogbn-arxiv & 71.85 & 72.17 & \textbf{72.43} & +0.26 $\pm$ 0.06 & 50.00 & 55.30 & \textbf{57.44} & +2.14 $\pm$ 0.94 \\
ogbn-products & 78.18 & 78.34 & \textbf{78.37} & +0.04 $\pm$ 0.11 & 39.56 & 39.61 & \textbf{39.69} & +0.07 $\pm$ 0.13 \\
Ele-Photo & 79.44 & 79.53 & \textbf{80.16} & +0.63 $\pm$ 0.26 & 76.38 & 76.81 & \textbf{77.43} & +0.62 $\pm$ 0.22 \\
Ele-Computers & 79.71 & 80.96 & \textbf{82.23} & +1.28 $\pm$ 0.30 & 73.70 & 77.80 & \textbf{79.92} & +2.12 $\pm$ 0.36 \\
Books-History & 81.86 & 81.85 & \textbf{81.90} & +0.05 $\pm$ 0.05 & 79.68 & 80.14 & \textbf{80.28} & +0.14 $\pm$ 0.14 \\
\midrule
Roman-Empire & \textbf{80.00} & 79.99 & 79.97 & -0.03 $\pm$ 0.05 & \textbf{21.05} & 21.01 & 21.01 & -0.00 $\pm$ 0.05 \\
Amazon-Ratings & 55.07 & 55.38 & \textbf{55.49} & +0.11 $\pm$ 0.15 & 33.29 & \textbf{37.84} & 37.64 & -0.20 $\pm$ 0.17 \\
\midrule
{Mean (9 main)} & {77.49} & {77.93} & {78.40} & {+0.47} & {63.39} & {66.21} & {67.44} & {+1.23} \\
\bottomrule
\end{tabular*}
\end{table}

With GraphSAGE, PtS achieves higher accuracy than APPNP on the nine main homophilic graphs at the selected noise levels \cref{tab:backbone-graphsage}. The gain varies across backbones and increases with noise level. For example, the improvement under noise is more prominent on ogbn-arxiv and Ele-Computers, while it is smaller on ogbn-products.

\section{Correct and Smooth comparison}
\label[appendix]{app:correct_smooth}
Correct \& Smooth uses known labels in its correction and smoothing stages. We evaluate sharpening within this label-assisted pipeline, leaving correction unchanged and modifying only smoothing. The C\&S-PtS variant tests whether the introduced sharpening adds gain over other methods that use APPNP propagation.

\begin{table}[H]
\centering
\small
\caption{Comparison of Correct and Smooth (C\&S) with C\&S-PtS, where the categorical sharpening step is applied after each smoothing step. All methods use an MLP backbone. Subscripts give the
corruption severity, and Gain is C\&S-PtS minus C\&S at $\sigma=2$ in percentage points.}
\label{tab:correct_smooth}

\begin{tabular}{@{}lrrrrr@{}}
\toprule
Dataset & C\&S$_0$ & C\&S-PtS$_0$ & C\&S$_2$ & C\&S-PtS$_2$ & Gain (pp) \\
\midrule
WikiCS         & 76.90 & 78.82 & 70.06 & 75.43 & +5.38 $\pm$ 0.85 \\
Cora-TAPE      & 87.38 & 87.23 & 86.18 & 86.00 & -0.17 $\pm$ 0.20 \\
PubMed-TAPE    & 86.97 & 86.98 & 84.44 & 84.44 & 0.00 $\pm$ 0.05 \\
TAPE-Arxiv23   & 70.32 & 70.31 & 45.54 & 45.55 & +0.02 $\pm$ 0.04 \\
ogbn-arxiv     & 70.91 & 71.00 & 68.35 & 68.61 & +0.26 $\pm$ 0.30 \\
ogbn-products  & 77.97 & 78.67 & 75.08 & 76.73 & +1.65 $\pm$ 0.53 \\
Ele-Photo      & 87.26 & 87.31 & 85.80 & 86.13 & +0.33 $\pm$ 0.09 \\
Ele-Computers  & 90.68 & 90.73 & 89.87 & 90.08 & +0.21 $\pm$ 0.08 \\
Books-History  & 84.53 & 84.53 & 82.74 & 82.81 & +0.07 $\pm$ 0.04 \\
\midrule
Roman-Empire   & 64.63 & 64.67 & 19.50 & 19.51 & +0.01 $\pm$ 0.07 \\
Amazon-Ratings & 53.57 & 53.53 & 45.76 & 46.20 & +0.44 $\pm$ 0.29 \\
\bottomrule
\end{tabular}
\end{table}

\Cref{tab:correct_smooth} shows the results on clean data and under noise, while \cref{tab:full-gaussian-ladder-cs} shows the entire noise ladder. The improvement is especially large on WikiCS and ogbn-products, while many other datasets show small differences. Sharpening can therefore also be effective here, but with dataset-dependent effects.

\section{Full Gaussian ladder}
\Cref{tab:full-gaussian-ladder} shows the results at all the five levels of gaussian noise. Each method chooses hyperparameters on the same noise levels that they are evaluated on.
\label[appendix]{app:ladder}

\begin{table}[H]
\centering
\caption{
Accuracy across the full Gaussian corruption ladder.
APPNP propagates logits, PPR-Prob isolates probability-space propagation,
and PtS additionally applies the categorical sharpening step.
}
\label{tab:full-gaussian-ladder}

\small
\renewcommand{\arraystretch}{1.08}
\setlength{\tabcolsep}{5pt}

\begin{tabular}{@{}llrrrrr@{}}
\toprule
Dataset & Method
& $\sigma=0$
& $\sigma=0.5$
& $\sigma=1$
& $\sigma=1.5$
& $\sigma=2$ \\
\midrule

WikiCS
& APPNP    & 77.78 & 77.12 & 75.33 & 72.19 & 67.97 \\
& PPR-Prob & 78.24 & 77.48 & 75.69 & 73.09 & 68.98 \\
& PtS      & \textbf{78.42} & \textbf{77.61} & \textbf{76.14} & \textbf{74.38} & \textbf{72.15} \\
\addlinespace[3pt]

Cora-TAPE
& APPNP    & 76.16 & 75.60 & 73.87 & 71.43 & 68.16 \\
& PPR-Prob & 78.42 & 77.92 & 76.28 & 73.52 & 70.47 \\
& PtS      & \textbf{79.74} & \textbf{78.95} & \textbf{77.32} & \textbf{75.21} & \textbf{72.30} \\
\addlinespace[3pt]

PubMed-TAPE
& APPNP    & 84.35 & 83.59 & 81.98 & 80.35 & 78.87 \\
& PPR-Prob & 84.96 & 84.17 & 82.76 & 81.46 & 80.22 \\
& PtS      & \textbf{85.18} & \textbf{84.39} & \textbf{83.05} & \textbf{81.84} & \textbf{80.73} \\
\addlinespace[3pt]

TAPE-Arxiv23
& APPNP    & 69.53 & \textbf{65.29} & \textbf{55.35} & 44.87 & 35.63 \\
& PPR-Prob & 69.68 & 65.27 & 55.32 & 45.87 & 38.26 \\
& PtS      & \textbf{69.72} & 65.28 & 55.33 & \textbf{45.90} & \textbf{38.91} \\
\addlinespace[3pt]

ogbn-arxiv
& APPNP    & 66.10 & 64.53 & 58.09 & 45.50 & 33.86 \\
& PPR-Prob & 65.96 & 65.03 & 62.72 & 52.24 & 40.43 \\
& PtS      & \textbf{66.56} & \textbf{65.92} & \textbf{63.04} & \textbf{53.93} & \textbf{40.62} \\
\addlinespace[3pt]

ogbn-products
& APPNP    & 70.58 & 56.83 & \textbf{38.51} & \textbf{31.64} & \textbf{28.54} \\
& PPR-Prob & 71.32 & \textbf{57.66} & 37.26 & 30.27 & 27.24 \\
& PtS      & \textbf{72.17} & 57.12 & 37.28 & 30.34 & 27.33 \\
\addlinespace[3pt]

Ele-Photo
& APPNP    & 71.87 & 70.39 & 66.70 & 61.25 & 55.97 \\
& PPR-Prob & 74.43 & 73.31 & 70.02 & 65.39 & 61.00 \\
& PtS      & \textbf{74.91} & \textbf{73.46} & \textbf{70.18} & \textbf{65.97} & \textbf{61.43} \\
\addlinespace[3pt]

Ele-Computers
& APPNP    & 75.37 & 74.14 & 71.17 & 66.25 & 59.94 \\
& PPR-Prob & 78.71 & 77.95 & 75.21 & 71.46 & 66.85 \\
& PtS      & \textbf{80.11} & \textbf{79.26} & \textbf{77.07} & \textbf{73.74} & \textbf{69.56} \\
\addlinespace[3pt]

Books-History
& APPNP    & 82.59 & 82.20 & 81.11 & 79.67 & 78.11 \\
& PPR-Prob & 82.87 & 82.47 & 81.41 & 80.21 & 78.90 \\
& PtS      & \textbf{82.88} & \textbf{82.49} & \textbf{81.47} & \textbf{80.31} & \textbf{79.13} \\

\midrule
\multicolumn{7}{@{}l}{\scriptsize{Heterophilic controls}} \\
\addlinespace[2pt]

Roman-Empire
& APPNP    & \textbf{65.65} & 53.84 & \textbf{36.19} & \textbf{25.11} & \textbf{19.30} \\
& PPR-Prob & 65.54 & \textbf{53.88} & \textbf{36.19} & 25.10 & \textbf{19.30} \\
& PtS      & 65.54 & 53.87 & 36.18 & 25.10 & \textbf{19.30} \\
\addlinespace[3pt]

Amazon-Ratings
& APPNP    & 51.63 & 44.20 & 38.70 & \textbf{37.09} & 36.74 \\
& PPR-Prob & \textbf{52.77} & \textbf{45.27} & \textbf{38.81} & 37.06 & \textbf{36.80} \\
& PtS      & 52.74 & 45.20 & \textbf{38.81} & 37.05 & 36.78 \\

\bottomrule
\end{tabular}
\end{table}

\begin{table}[H]
\centering
\caption{
Correct \& Smooth across the full Gaussian corruption ladder.
C\&S-PtS applies the same categorical sharpening step during the smoothing stage.
}
\label{tab:full-gaussian-ladder-cs}

\small
\renewcommand{\arraystretch}{1.08}
\setlength{\tabcolsep}{5pt}

\begin{tabular}{@{}llrrrrr@{}}
\toprule
Dataset & Method
& $\sigma=0$
& $\sigma=0.5$
& $\sigma=1$
& $\sigma=1.5$
& $\sigma=2$ \\
\midrule

WikiCS
& C\&S     & 76.90 & 75.64 & 73.14 & 71.19 & 70.06 \\
& C\&S-PtS & \textbf{78.82} & \textbf{78.27} & \textbf{77.32} & \textbf{76.31} & \textbf{75.43} \\
\addlinespace[3pt]

Cora-TAPE
& C\&S     & \textbf{87.38} & \textbf{87.14} & \textbf{86.81} & \textbf{86.54} & \textbf{86.18} \\
& C\&S-PtS & 87.23 & 87.05 & 86.66 & 86.42 & 86.00 \\
\addlinespace[3pt]

PubMed-TAPE
& C\&S     & 86.97 & \textbf{86.39} & \textbf{85.48} & \textbf{84.89} & \textbf{84.44} \\
& C\&S-PtS & \textbf{86.98} & 86.38 & \textbf{85.48} & 84.88 & \textbf{84.44} \\
\addlinespace[3pt]

TAPE-Arxiv23
& C\&S     & \textbf{70.32} & 66.31 & 57.91 & \textbf{50.56} & 45.54 \\
& C\&S-PtS & 70.31 & \textbf{66.32} & \textbf{57.92} & \textbf{50.56} & \textbf{45.55} \\
\addlinespace[3pt]

ogbn-arxiv
& C\&S     & 70.91 & 69.99 & 68.98 & 68.49 & 68.35 \\
& C\&S-PtS & \textbf{71.00} & \textbf{70.07} & \textbf{69.16} & \textbf{68.71} & \textbf{68.61} \\
\addlinespace[3pt]

ogbn-products
& C\&S     & 77.97 & 76.55 & 75.84 & 75.26 & 75.08 \\
& C\&S-PtS & \textbf{78.67} & \textbf{78.07} & \textbf{77.25} & \textbf{77.04} & \textbf{76.73} \\
\addlinespace[3pt]

Ele-Photo
& C\&S     & 87.26 & 87.04 & 86.62 & 86.20 & 85.80 \\
& C\&S-PtS & \textbf{87.31} & \textbf{87.11} & \textbf{86.73} & \textbf{86.40} & \textbf{86.13} \\
\addlinespace[3pt]

Ele-Computers
& C\&S     & 90.68 & 90.57 & 90.35 & 90.08 & 89.87 \\
& C\&S-PtS & \textbf{90.73} & \textbf{90.65} & \textbf{90.47} & \textbf{90.26} & \textbf{90.08} \\
\addlinespace[3pt]

Books-History
& C\&S     & \textbf{84.53} & \textbf{84.33} & 83.78 & 83.24 & 82.74 \\
& C\&S-PtS & \textbf{84.53} & 84.32 & \textbf{83.84} & \textbf{83.27} & \textbf{82.81} \\

\midrule
\multicolumn{7}{@{}l}{\scriptsize\textit{Heterophilic controls}} \\
\addlinespace[2pt]

Roman-Empire
& C\&S     & 64.63 & \textbf{53.00} & \textbf{35.98} & \textbf{25.22} & 19.50 \\
& C\&S-PtS & \textbf{64.67} & 52.96 & 35.95 & \textbf{25.22} & \textbf{19.51} \\
\addlinespace[3pt]

Amazon-Ratings
& C\&S     & \textbf{53.57} & \textbf{48.18} & 46.47 & 45.98 & 45.76 \\
& C\&S-PtS & 53.53 & \textbf{48.18} & \textbf{46.74} & \textbf{46.33} & \textbf{46.20} \\

\bottomrule
\end{tabular}
\end{table}

\newpage

%
%
\section{Preserving mass}
The row mass chooses how strongly a node's class distribution influences the next propagation step. In \cref{tab:mass-preservation}, we test whether rescaling this mass helps after sharpening, compared with normalizing each row sum to one. The comparison also includes propagation without sharpening so that we can examine the effect of row normalization in both cases.

\label[appendix]{app:mass}
We compare PtS with a variant that resets each row's mass to one after every iteration, $R_\eta^{\mathrm{rn}}(h)=r_\eta(h/m(h))$. Its sharpening-off counterpart, PPR-rn, applies row normalisation after every probability propagation step.
\begin{table}[H]
\centering\small
\caption{Effect of preserving the propagated row mass. $\Delta_{\mathrm{PtS}}$ is
PtS minus PtS-rn and $\Delta_{\mathrm{PPR}}$ is PPR-Prob minus PPR-rn, both computed
within (split, seed, draw). Bold marks the higher of the two accuracies in each row.}
\label{tab:mass-preservation}
\begin{tabular}{@{}lrrrr@{}}
\toprule
Dataset & PtS-rn & PtS 
        & $\Delta_{\mathrm{PtS}}$
        & $\Delta_{\mathrm{PPR}}$ \\
\midrule
\multicolumn{5}{@{}l}{\textbf{Clean ($\sigma=0$)}} \\
\addlinespace[3pt]
WikiCS
& 78.24 & \textbf{78.42}
& $+0.18 \pm 0.27$ & $+0.15 \pm 0.27$ \\
Cora-TAPE
& 79.15 & \textbf{79.74}
& $+0.59 \pm 0.66$ & $+0.17 \pm 0.58$ \\
PubMed-TAPE
& 85.12 & \textbf{85.18}
& $+0.07 \pm 0.16$ & $+0.03 \pm 0.08$ \\
TAPE-Arxiv23
& \textbf{69.75} & 69.72
& $-0.03 \pm 0.06$ & $-0.08 \pm 0.04$ \\
ogbn-arxiv
& \textbf{67.45} & 66.56
& $-0.89 \pm 0.34$ & $-1.11 \pm 0.32$ \\
ogbn-products
& 71.93 & \textbf{72.17}
& $+0.24 \pm 0.11$ & $+0.60 \pm 0.05$ \\
Ele-Photo
& \textbf{74.92} & 74.91
& $-0.01 \pm 0.30$ & $-0.34 \pm 0.47$ \\
Ele-Computers
& 79.67 & \textbf{80.11}
& $+0.45 \pm 0.15$ & $+0.01 \pm 0.33$ \\
Books-History
& \textbf{82.90} & 82.88
& $-0.02 \pm 0.17$ & $+0.02 \pm 0.07$ \\
\addlinespace[3pt]
\textbf{Mean (9 main)}
& 76.57 & \textbf{76.63}
& $\mathbf{+0.06}$ & $\mathbf{-0.06}$ \\
\addlinespace[5pt]
\multicolumn{5}{@{}l}{Heterophilic controls} \\
\addlinespace[2pt]
Roman-Empire
& \textbf{65.54} & \textbf{65.54}
& $+0.00 \pm 0.02$ & $+0.00 \pm 0.01$ \\
Amazon-Ratings
& 52.69 & \textbf{52.74}
& $+0.05 \pm 0.14$ & $+0.02 \pm 0.15$ \\

\midrule
\multicolumn{5}{@{}l}{\textbf{Corrupted ($\sigma=2$)}} \\
\addlinespace[3pt]
WikiCS
& 69.66 & \textbf{72.15}
& $+2.48 \pm 0.42$ & $+1.01 \pm 0.42$ \\
Cora-TAPE
& 71.24 & \textbf{72.30}
& $+1.06 \pm 0.60$ & $+0.50 \pm 0.46$ \\
PubMed-TAPE
& 80.63 & \textbf{80.73}
& $+0.10 \pm 0.13$ & $+0.14 \pm 0.17$ \\
TAPE-Arxiv23
& 38.69 & \textbf{38.91}
& $+0.22 \pm 0.16$ & $+0.19 \pm 0.10$ \\
ogbn-arxiv
& \textbf{41.68} & 40.62
& $-1.07 \pm 0.29$ & $-0.17 \pm 1.11$ \\
ogbn-products
& \textbf{27.44} & 27.33
& $-0.12 \pm 0.14$ & $-0.26 \pm 0.26$ \\
Ele-Photo
& 60.73 & \textbf{61.43}
& $+0.70 \pm 0.55$ & $+0.40 \pm 0.32$ \\
Ele-Computers
& 67.42 & \textbf{69.56}
& $+2.14 \pm 0.37$ & $+1.10 \pm 0.49$ \\
Books-History
& 79.09 & \textbf{79.13}
& $+0.04 \pm 0.09$ & $+0.04 \pm 0.06$ \\
\addlinespace[3pt]
\textbf{Mean (9 main)}
& 59.62 & \textbf{60.24}
& $\mathbf{+0.61}$ & $\mathbf{+0.33}$ \\
\addlinespace[5pt]
\multicolumn{5}{@{}l}{Heterophilic controls} \\
\addlinespace[2pt]
Roman-Empire
& 19.29 & \textbf{19.30}
& $+0.00 \pm 0.04$ & $+0.00 \pm 0.03$ \\
Amazon-Ratings
& \textbf{36.79} & 36.78
& $-0.02 \pm 0.07$ & $-0.01 \pm 0.06$ \\
\bottomrule
\end{tabular}
\end{table}

The differences are small on many noise-free datasets while mass preservation gives larger improvements under noise, especially on WikiCS and Ele-Computers. The effect is still not positive everywhere, on ogbn-arxiv, the variant with row normalization is better in both. This means mass preservation influences the results but does not guarantee higher accuracy.

\section{Depth sensitivity}
\label{app:depth}
\Cref{tab:depth-values} gives the complete values for the depth analysis in \cref{fig:depth}. Logit-Sharp updates the logits as $L\leftarrow L+\eta\,\softmax(L)$ after each APPNP propagation step, with final predictions $\softmax(L)$. We vary propagation steps $K$, while keeping $\alpha$ and $\eta$ fixed. Without restart, APPNP and PPR-Prob are pure diffusion methods, so they oversmooth and lose substantial accuracy at greater depths, while strong sharpening reduces this loss substantially. Restart dampens this decrease, but PtS keeps its advantage at greater depths here. A larger $\eta$ does not always help: $\eta =200$ is best at large depth without restart, while $\eta = 16$ is better with restart at $\sigma=2$.

\begin{table}[H]
\centering
\caption{Depth sensitivity under clean and corrupted inputs.}
\label{tab:depth-values}
\begin{tabular}{@{}lrrrrrrrr@{}}
\toprule
Curve & $K{=}1$ & $K{=}2$ & $K{=}3$ & $K{=}5$ & $K{=}10$ & $K{=}20$ & $K{=}40$ & $K{=}100$ \\
\midrule
\multicolumn{9}{@{}l@{}}{\textit{$\sigma=0$, $\alpha=0$ (no restart)}} \\
frozen $Q$             & 67.5 & 67.5 & 67.5 & 67.5 & 67.5 & 67.5 & 67.5 & 67.5 \\
APPNP                  & 74.5 & 74.3 & 73.4 & 70.7 & 64.8 & 58.2 & 50.2 & 40.5 \\
Logit-Sharp $\eta{=}16$ & 74.5 & 74.9 & 74.7 & 74.3 & 72.7 & 68.0 & 60.7 & 52.9 \\
PPR-Prob               & 75.1 & 75.8 & 75.2 & 73.3 & 68.5 & 61.4 & 53.7 & 44.6 \\
PtS $\eta{=}16$        & 75.1 & 75.8 & 75.8 & 75.3 & 74.4 & 73.5 & 71.9 & 70.6 \\
PtS $\eta{=}200$       & 75.1 & 75.4 & 75.4 & 75.1 & 74.7 & 74.2 & 73.5 & 73.2 \\
\addlinespace[2pt]
\multicolumn{9}{@{}l@{}}{\textit{$\sigma=0$, $\alpha=0.1$ (restart)}} \\
frozen $Q$             & 67.5 & 67.5 & 67.5 & 67.5 & 67.5 & 67.5 & 67.5 & 67.5 \\
APPNP                  & 74.3 & 74.5 & 74.1 & 73.0 & 71.4 & 70.6 & 70.4 & 70.4 \\
Logit-Sharp $\eta{=}16$ & 74.3 & 75.1 & 75.0 & 74.6 & 73.3 & 70.2 & 64.6 & 61.0 \\
PPR-Prob               & 74.8 & 75.6 & 75.6 & 75.0 & 74.0 & 73.6 & 73.5 & 73.5 \\
PtS $\eta{=}16$        & 74.8 & 75.9 & 76.1 & 75.8 & 75.3 & 74.9 & 74.5 & 74.2 \\
PtS $\eta{=}200$       & 74.8 & 75.7 & 75.7 & 75.5 & 75.2 & 75.0 & 74.9 & 74.8 \\
\addlinespace[2pt]
\multicolumn{9}{@{}l@{}}{\textit{$\sigma=2$, $\alpha=0$ (no restart)}} \\
frozen $Q$             & 41.9 & 41.9 & 41.9 & 41.9 & 41.9 & 41.9 & 41.9 & 41.9 \\
APPNP                  & 54.0 & 55.8 & 55.5 & 53.8 & 50.3 & 45.9 & 40.4 & 34.7 \\
Logit-Sharp $\eta{=}16$ & 54.0 & 56.0 & 56.4 & 56.4 & 55.6 & 53.0 & 48.5 & 42.5 \\
PPR-Prob               & 54.2 & 57.7 & 58.5 & 57.8 & 54.4 & 50.2 & 44.5 & 37.2 \\
PtS $\eta{=}16$        & 54.2 & 58.1 & 58.8 & 59.4 & 59.0 & 58.1 & 57.0 & 56.1 \\
PtS $\eta{=}200$       & 54.2 & 57.7 & 58.0 & 58.1 & 58.7 & 58.5 & 57.9 & 57.7 \\
\addlinespace[2pt]
\multicolumn{9}{@{}l@{}}{\textit{$\sigma=2$, $\alpha=0.1$ (restart)}} \\
frozen $Q$             & 41.9 & 41.9 & 41.9 & 41.9 & 41.9 & 41.9 & 41.9 & 41.9 \\
APPNP                  & 53.5 & 55.2 & 55.4 & 54.7 & 53.5 & 52.8 & 52.6 & 52.6 \\
Logit-Sharp $\eta{=}16$ & 53.5 & 55.8 & 56.4 & 56.5 & 56.0 & 54.2 & 51.1 & 48.0 \\
PPR-Prob               & 53.1 & 55.6 & 56.3 & 56.3 & 55.5 & 55.0 & 54.8 & 54.8 \\
PtS $\eta{=}16$        & 53.1 & 57.1 & 57.9 & 58.1 & 58.2 & 57.8 & 57.3 & 57.0 \\
PtS $\eta{=}200$       & 53.1 & 56.8 & 57.2 & 57.1 & 56.7 & 56.4 & 56.2 & 56.1 \\
\bottomrule
\end{tabular}
\end{table}

\section{Dirichlet and Gini}
\Cref{fig:energy_evolve} shows how the energy components and validation accuracy evolve through the PtS iteration on WikiCS. We use the configuration $\alpha = 0.1$, $\eta = 16$, and $\sigma = 2$ to illustrate what happens before and after propagation and sharpening.

\begin{figure}[H]
    \centering
    \includegraphics[width=1\linewidth]{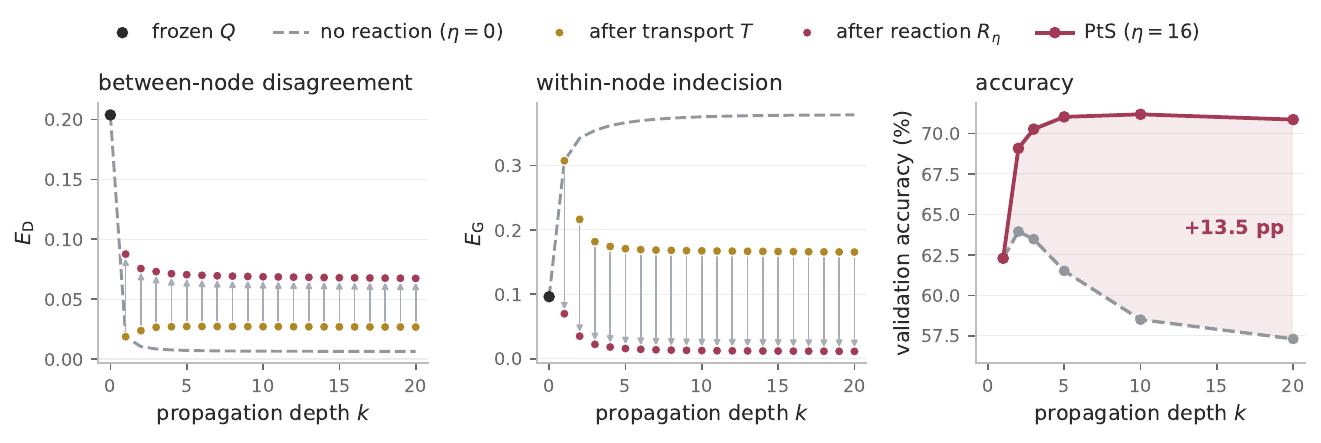}
    \caption{Evolution of the Dirichlet energy $E_{\mathrm D}$ (left), Gini energy $E_{\mathrm G}$ (middle), and validation accuracy (right) under gaussian feature corruption at $\sigma =2$. Gold and red markers show values after propagation and sharpening. Black markers denote the frozen predictions $Q$. Propagation  reduces disagreement between nodes, but increases indecision within nodes, while sharpening has the opposite effect}
    \label{fig:energy_evolve}
\end{figure}

In the visualization, we can see that propagation reduces the Dirichlet term and increases the Gini term, while sharpening has the opposite effect. Still, PtS maintains higher validation accuracy at greater depth than the variant without reaction. Since sharpening preserves each node's predicted class, its effect must occur through subsequent propagation steps. \Cref{prop:descent} guarantees local energy descent under sharpening.

\section{Per-node performance}
\label[appendix]{app:per_node}
We start from a strong homophily prior: neighbouring nodes should be similar. We group nodes by local homophily, degree quintile, and by the confidence $\max_c Q_{ic}$. \Cref{tab:per_node_performance} shows the test-accuracy difference between PtS and APPNP in each group, using backbone seed $0$ for each split and noise draw $0$ at $\sigma=2$, with the hyperparameters selected in the main experiment.

We can see that PtS gives the greatest improvement in groups with higher local homophily. Most interestingly, PtS helps most in uncertain areas, the accuracy increase is also higher at nodes with lower prediction confidence. This suggests that PtS is most useful when the original predictions are uncertain, and the neighbourhood provides relevant class information.

\begin{table}[H]
\centering
\caption{
PtS$-$APPNP accuracy difference in percentage points, with nodes stratified by local
homophily, by degree quintile, and by the confidence of the frozen prediction.
}
\label{tab:per_node_performance}

\small
\renewcommand{\arraystretch}{1.08}
\setlength{\tabcolsep}{5pt}

\begin{tabular}{@{}llrr@{}}
\toprule
Node stratum & Bucket 
& PtS$-$APPNP ($\sigma{=}0$) 
& PtS$-$APPNP ($\sigma{=}2$) \\
\midrule
Local homophily 
 & 0-.2 & -1.29 & -1.75 \\
 & .2-.4 & -0.51 & -0.35 \\
 & .4-.6 & +1.88 & +4.83 \\
 & .6-.8 & +3.69 & +9.12 \\
 & .8-1 & +1.90 & +5.76 \\
 & isolated & +0.00 & +0.00 \\
\addlinespace[2pt]
Degree quintile & Q1 & +1.27 & +3.64 \\
 & Q2 & +1.63 & +3.56 \\
 & Q3 & +1.86 & +3.77 \\
 & Q4 & +1.95 & +4.41 \\
 & Q5 & +2.04 & +5.66 \\
\addlinespace[2pt]
Anchor confidence quintile & Q1 & +3.41 & +5.12 \\
 & Q2 & +2.30 & +4.66 \\
 & Q3 & +1.60 & +4.39 \\
 & Q4 & +0.95 & +3.87 \\
 & Q5 & +0.48 & +3.00 \\
\bottomrule
\end{tabular}
\end{table}

\newpage

%
%
\section{Runtime and overhead}
\label[appendix]{app:runtime}
\Cref{tab:runtime} shows the runtime for the base models forward pass, as well as the post hoc processing at both 1 and 100 steps. We can see that, generally, for all methods, the runtime is very low.

\begin{table}[H]
\centering
\small
\setlength{\tabcolsep}{3pt}
\renewcommand{\arraystretch}{1.08}
\caption{Runtime and overhead: MLP forward-pass and post-hoc call
times in milliseconds at $\sigma=2$}
\label{tab:runtime}

\resizebox{\linewidth}{!}{%
\begin{tabular}{@{}llrrrrrrrr@{}}
\toprule
& & MLP
& \multicolumn{3}{c}{Per step (ms)}
& \multicolumn{3}{c}{$K=100$ (ms)}
& PtS/APPNP \\
\cmidrule(lr){4-6}\cmidrule(lr){7-9}
Dataset & Device & fwd (ms)
& APPNP & PPR-Prob & PtS
& APPNP & PPR-Prob & PtS
& per step \\
\midrule
WikiCS         & CPU  & 23.0 & 1.00 & 2.02 & 4.72 & 100   & 202   & 472   & 4.7 \\
Cora-TAPE      & CPU  & 5.4  & 0.09 & 0.44 & 1.31 & 8.8   & 43.6  & 131   & 14.9 \\
PubMed-TAPE    & CPU  & 59.0 & 0.18 & 0.90 & 3.08 & 17.7  & 89.8  & 308   & 17.4 \\
TAPE-Arxiv23   & CPU  & 87.6 & 5.45 & 12.1 & 30.7 & 545   & 1,211 & 3,068 & 5.6 \\
ogbn-arxiv     & A100 & 9.9  & 0.90 & 1.69 & 3.85 & 90.0  & 169   & 385   & 4.3 \\
ogbn-products  & A100 & 161  & 30.7 & 36.7 & 54.3 & 3,072 & 3,669 & 5,428 & 1.8 \\
Ele-Photo      & CPU  & 169  & 1.98 & 3.87 & 9.73 & 198   & 387   & 973   & 4.9 \\
Ele-Computers  & CPU  & 217  & 4.03 & 6.55 & 15.6 & 403   & 655   & 1,564 & 3.9 \\
Books-History  & CPU  & 138  & 1.98 & 3.29 & 9.19 & 198   & 329   & 919   & 4.6 \\
\midrule
Roman-Empire   & CPU  & 41.7 & 0.41 & 1.64 & 5.61 & 40.8  & 164   & 561   & 13.8 \\
Amazon-Ratings & CPU  & 45.2 & 0.28 & 0.93 & 3.10 & 27.8  & 92.7  & 310   & 11.2 \\
\bottomrule
\end{tabular}%
}
\end{table}

\section{Calibration}
\label[appendix]{app:calibration}
Sharpening can make class distributions more concentrated, but directly influencing probabilities can harm the model's calibration. We therefore examine the negative log-likelihood (NLL) and expected calibration error (ECE) before and after temperature scaling using validation nodes. 

\begin{table}[H]
\centering\small
\caption{Calibration of the frozen prediction, APPNP and PtS with an MLP backbone, before
and after temperature scaling fitted on validation nodes. $T$ is the fitted temperature and
accuracy drift is the change in accuracy caused by scaling. Values are means over WikiCS, Cora-TAPE, PubMed-TAPE, TAPE-arxiv23, ogbn-arxiv and ogbn-products}
\label{calibration}
\begin{tabular}{@{}llrrrrrr@{}}
\toprule
$\sigma$ & Method & $T$ & raw NLL & scaled NLL & raw ECE & scaled ECE & acc.\ drift (pp) \\
\midrule
0 & Frozen $Q$ & 1.34 & 1.226 & 1.132 & 0.088 & 0.040 & 0 \\
 & APPNP & 0.84 & 0.983 & 0.907 & 0.129 & 0.037 & 0 \\
 & PtS & 2.11 & 1.591 & 0.932 & 0.158 & 0.050 & 0 \\
2 & Frozen $Q$ & 3.55 & 3.514 & 2.112 & 0.302 & 0.047 & 0 \\
 & APPNP & 1.47 & 2.108 & 1.752 & 0.178 & 0.075 &0 \\
 & PtS & 2.92 & 3.088 & 1.692 & 0.254 & 0.093 & 0 \\
\bottomrule
\end{tabular}
\end{table}
PtS has higher uncorrected NLL and ECE than APPNP, both with and without noise. Temperature scaling mitigates this without affecting accuracy.

\end{document}